\pdfoutput=1
\documentclass[11pt]{article}
\usepackage{float}
\usepackage[textsize=tiny]{todonotes}
\usepackage{fullpage,parskip}
\usepackage[english]{babel}
\usepackage[utf8]{inputenc}
\usepackage[T1]{fontenc}
\usepackage{enumitem,bbm,ifthen}
\usepackage[authoryear,round]{natbib}
\usepackage{amsmath,amssymb,amsthm,mathtools,bm}
\usepackage{thm-restate}
\usepackage{booktabs,array}
\usepackage{algorithm,algpseudocode}
\usepackage{graphicx}
\usepackage{tikz}
\usetikzlibrary{decorations.pathreplacing,calc}
\usepackage{subcaption}
\usepackage{microtype}
\usepackage{xcolor}
\usepackage{hyperref}
\usepackage{url}

\hypersetup{colorlinks=true,citecolor=blue,linkcolor=blue,urlcolor=blue}

\newcommand{\R}{\mathbb{R}}
\newcommand{\N}{\mathbb{N}}
\newcommand{\E}{\mathbb{E}}
\newcommand{\Pp}{\mathbb{P}}
\newcommand{\Law}{\operatorname{Law}}
\newcommand{\TV}{\operatorname{TV}}
\newcommand{\KL}{D_{\rm KL}}
\newcommand{\Unif}{\mathcal{U}}
\newcommand{\excessloss}{\Delta}
\newcommand{\avgexcess}{A}
\newcommand{\riskdecay}{\Psi}
\newcommand{\Normal}{\mathcal{N}}

\newcommand{\norm}[1]{\left\lVert #1\right\rVert}
\newcommand{\norms}[1]{\lVert #1\rVert}
\newcommand{\ip}[2]{\left\langle #1,#2\right\rangle}
\newcommand{\dd}{\,\mathrm{d}}

\newcommand{\ma}{{\mathbf{A}}}
\newcommand{\mb}{{\mathbf{B}}}
\newcommand{\mE}{{\mathbf{E}}}

\newcommand{\mR}{{\mathbf{R}}}
\newcommand{\vtheta}{{\boldsymbol{\theta}}}

\newcommand{\Lfm}{\widehat L_{\rm FM}}

\newcommand{\Leeg}{\widehat L_{\rm EEG}}
\newcommand{\Lenigma}{\widehat L_{\rm ENIGMA}}

\newcommand{\mq}{{\mathbf{Q}}}
\newcommand{\mw}{{\mathbf{W}}}
\newcommand{\mSigma}{{\boldsymbol{\Sigma}}}
\newcommand{\va}{{\mathbf{a}}}
\newcommand{\vb}{{\mathbf{b}}}
\newcommand{\vr}{{\mathbf{r}}}
\newcommand{\vf}{{\mathbf{f}}}
\newcommand{\vg}{{\mathbf{g}}}
\newcommand{\vh}{{\mathbf{h}}}
\newcommand{\vm}{{\mathbf{m}}}
\newcommand{\vo}{{\mathbf{o}}}
\newcommand{\vu}{{\mathbf{u}}}
\newcommand{\vv}{{\mathbf{v}}}
\newcommand{\vw}{{\mathbf{w}}}
\newcommand{\vx}{{\mathbf{x}}}
\newcommand{\vy}{{\mathbf{y}}}
\newcommand{\vz}{{\mathbf{z}}}
\newcommand{\vzero}{{\mathbf{0}}}
\newcommand{\Id}{{\mathbf{I}}}

\newtheorem{theorem}{Theorem}
\newtheorem{lemma}{Lemma}

\newtheorem{proposition}{Proposition}
\newtheorem{corollary}{Corollary}
\theoremstyle{definition}

\newtheorem{model}{Model}
\newtheorem{problem}{Problem}
\newtheorem{assumption}{Assumption}
\theoremstyle{remark}
\newtheorem*{remark}{Remark}

\providecommand{\ma}{{\mathbf{A}}}
\providecommand{\mb}{{\mathbf{B}}}

\providecommand{\mq}{{\mathbf{Q}}}

\providecommand{\mw}{{\mathbf{W}}}

\providecommand{\va}{{\mathbf{a}}}
\providecommand{\vb}{{\mathbf{b}}}

\providecommand{\vf}{{\mathbf{f}}}
\providecommand{\vg}{{\mathbf{g}}}
\providecommand{\vh}{{\mathbf{h}}}

\providecommand{\vm}{{\mathbf{m}}}

\providecommand{\vo}{{\mathbf{o}}}

\providecommand{\vr}{{\mathbf{r}}}

\providecommand{\vu}{{\mathbf{u}}}
\providecommand{\vv}{{\mathbf{v}}}
\providecommand{\vw}{{\mathbf{w}}}
\providecommand{\vx}{{\mathbf{x}}}
\providecommand{\vy}{{\mathbf{y}}}
\providecommand{\vz}{{\mathbf{z}}}

\title{Conditional Flow Matching for\\Transport Between Markov Processes}

\author{
\normalsize Syamantak Kumar\textsuperscript{\ensuremath{\dagger}}\quad
Dheeraj Nagaraj\textsuperscript{*}\quad
Saptarshi Roy\textsuperscript{\ensuremath{\ddagger}}\quad
Purnamrita Sarkar\textsuperscript{\ensuremath{\mathsection}}
}

\date{}

\begin{document}
\allowdisplaybreaks
\hypersetup{pageanchor=false}
\pagenumbering{gobble}
\maketitle
\begingroup
\renewcommand{\thefootnote}{\fnsymbol{footnote}}
\footnotetext[1]{Google Deepmind, \texttt{dheerajnagaraj@google.com}. Dheeraj Nagaraj was involved in an Advisory Role.}
\footnotetext[2]{The University of Texas at Austin, Department of Computer Science, \texttt{syamantak@utexas.edu}.}
\footnotetext[3]{Mohamed bin Zayed University of Artificial Intelligence, CMS Division, \texttt{Saptarshi.Roy@mbzuai.ac.ae}.}
\footnotetext[4]{The University of Texas at Austin, \texttt{purna.sarkar@utexas.edu}.}
\endgroup

\begin{abstract}
Motivated by sequence-to-sequence transport in the context time-series domain adaptation, we study the problem of transportation between trajectories of Markov processes. Given a limited number of trajectories from source distribution and the target distribution, we formulate a flow matching based algorithm which learns a transport map from the source to target trajectory distribution, while preserving the Markov structure. We show that this is consistent in the population limit and derive finite-sample error bounds under mixing time assumptions, following the analysis of classical statistical problems including
regression \citep{nagaraj2020least}, principal component analysis
\citep{kumar2023streaming}, and matrix
concentration \citep{neeman2024concentration} in the Markov setting. We complement that with a lower-bound construction showing that a mixing-time dependent sample complexity is unavoidable even with regular Gaussian conditional transitions. We evaluate on synthetic
and real-world data. For image retrieval from electroencephalography (EEG) on THINGS-EEG2
\citep{gifford2022things}, the task is to identify the
viewed image from EEG signals captured from human subjects, which suffers from high inter subject variability. We augment the ENIGMA decoder \citep{kneeland2026enigma} with a conditional flow before its subject-specific temporal map. This improves mean top-5 retrieval accuracy from $43.87\%$ to $49.05\%$, an $11.82\%$ relative improvement.

% We show that adding our width-64 flow before the temporal map of the
% reproduced ENIGMA decoder \citep{kneeland2026enigma}, with $k=32$, raises mean top-5 retrieval accuracy from $43.87\%$ to
% $49.05\%$ under a matched subject-adaptation protocol.

% Given one
% source and one target trajectory, we learn to transform a fresh source
% trajectory into one following the target dynamics. We use
% conditional flow matching to learn a transition map conditioned on the current
% source and generated target states and reuse it recursively. At the population
%level, we prove that the resulting trajectory has exactly the target %trajectory distribution. 

\end{abstract}

\thispagestyle{empty}
\newpage
\tableofcontents
\thispagestyle{empty}
\newpage
\pagenumbering{arabic}
\hypersetup{pageanchor=true}

\section{Introduction}\label{sec:introduction}

Sequence-to-sequence tasks seek to learn a map from an input sequence to an output sequence. In time-series adaptation, such a map can transform observations from a source process to match the statistical behavior of a target process. Related transport and alignment problems arise in sensor calibration \citep{kim2020adaptive}, climate-model bias correction \citep{robin2019multivariate}, and time-series domain adaptation \citep{painblanc2023match,mishra2026causal}.

In many of these cases, the sequences themselves are long, but the number of
sequences available from the source and target distributions is limited.
A practical concern is how much data must be collected to learn this map for
a subject or device. Sensor calibration can require repeated, costly
data collection \citep{kim2020adaptive}, while neural decoding motivates
adapting pretrained models using limited recordings from each new subject.
ENIGMA and MindEye2 pursue this objective for EEG and functional magnetic
resonance imaging (fMRI), respectively
\citep{kneeland2026enigma,scotti2024mindeye2}. These constraints motivate
studying adaptation from one or a few observed trajectories. Although each
trajectory contains many observations, those observations are temporally
dependent. This raises the question:

\textit{Can we learn, from \textbf{one} source and \textbf{one} target trajectory, a map that
transforms a fresh source trajectory into one that follows the target dynamics?}

In this paper, we study this problem for Markovian trajectories. When both
the source and target follow Markov processes, we want the learned mapping
to preserve the target's transition law and hence its Markov structure in
the generated trajectory. At first glance, this task seems daunting:
treating an entire trajectory as one high-dimensional datapoint leaves us
with only one sample from each trajectory distribution. However, the Markov
property and time homogeneity let us learn a common transition rule from
successive observations. Under stationarity, these observations share a
common transition-pair distribution, although they remain dependent.
We exploit this structure to give a principled answer to the question above.

\textit{In particular, we propose an algorithm based on conditional flow
matching for transporting a source trajectory to a target trajectory,
accounting for the source history through the current source state.
We prove exact population-level transport under suitable assumptions,
provide finite-sample error bounds, and evaluate the algorithm on real
and simulated data.}

Transport between Markov chains has also been studied through
cost-minimizing joint distributions over transitions and trajectories
\citep{oconnor2022optimal,moulos2021bicausal}. Our objective is to learn a
recursively reusable generator from one source and one target trajectory.

Formally, let $(\vx_i)_{i\geq0}$ and $(\vy_i)_{i\geq0}$ be independent, time-homogeneous Markov
processes on $\R^d$ with transition kernels $P_X$ and $P_Y$. We observe trajectories
$\vx_{0:m}$ and $\vy_{0:m}$. At test time, we receive a fresh source trajectory
$\vx'_{0:n}$ with transition kernel $P_X$ and seek to map it to a trajectory
$\widehat{\vy}_{0:n}$ with transition kernel $P_Y$. Given the source trajectory, we map each next source state conditional on the current source and generated target states.
Lemma~\ref{lem:current-state-sufficiency} shows that, conditional on their
joint history, the next source state still follows the source transition
kernel. Thus maps that transport each source transition to the corresponding
target transition can be applied recursively. We learn a single family of maps
$\Phi_{\vx,\vy}:\R^d\to\R^d$ and generate recursively by
\begin{equation}
 \widehat{\vy}_{i+1}
 =\Phi_{\vx'_i,\widehat{\vy}_i}(\vx'_{i+1}).
 \label{eq:intro-transition-goal}
\end{equation}

Natural heuristics such as Histogram or marginal matching matches state frequencies alone and need not preserve the transition law.
For e.g., consider two Markov processes on $\{0,1\}$ that are equally likely
to be in either state. If the source repeats its current state with probability
$0.9$ while the target switches with probability $0.9$, their state frequencies
agree, with very different transition behaviour.

% \paragraph{Our approach.}
% \red{\textbf{TODO(Purna):}
% Learning an unrestricted transport between trajectories by treating entire trajectories as datapoints would provide us with only
% one very high-dimensional sample, i.e the entire trajectory, from each of the source and target trajectory distribution. However, practical trajectory distributions reveal a lot of hidden structure that we can exploit by using the dynamics of the trajectory. Particularly, we use the Markov property and time
% homogeneity to split each trajectory into multiple observations from a common transition-pair
% distribution, which provides us with enough samples to learn the inherent transport maps efficiently. On THINGS-EEG2, our flow raises ENIGMA's mean Top-5 retrieval accuracy from $43.87\%$ to $49.05\%$.}

\paragraph{Our approach.} To learn the transition map, we use conditional flow matching (CFM)
\citep{lipman2023flow,liu2023flow}, regressing the
next-state displacement $\vy^+-\vx^+$ on
$\vz_t=(1-t)\vx^++t\vy^+$, with flow time $t\in[0,1]$ and current states
$(\vx,\vy)$ as conditioning inputs. Here
$\vx^+\sim P_X(\vx,\cdot)$ and $\vy^+\sim P_Y(\vy,\cdot)$ are independent
given the current states. At test time, integrating the learned velocity
from $\vx'_{i+1}$ with $(\vx'_i,\widehat{\vy}_i)$ fixed implements
the recursive update \eqref{eq:intro-transition-goal}.
Our analysis quantifies how transition mixing time affects the required training trajectory length.
Section~\ref{sec:setup} defines the conditional velocity and its ODE map,
and Figure~\ref{fig:overview} summarizes the algorithm.

\begin{figure}[H]
 \centering
 \includegraphics[width=\linewidth]{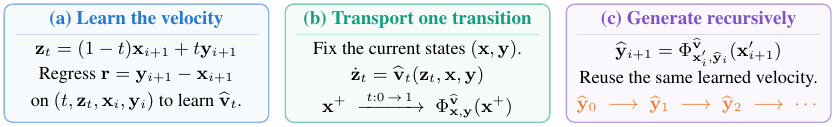}
 \caption{Each learned flow transports one source transition to one target
 transition. The index $i$ denotes time in the Markov process, while
 $t\in[0,1]$ denotes flow time within a single transition.}
 \label{fig:overview}
\end{figure}

\subsection{Our results}\label{sec:intro-results}

\textbf{Population correctness.}
Propositions~\ref{prop:conditional-transport} and~\ref{prop:exact-trajectory}
prove exact generation under suitable assumptions.
The population flow maps each source transition $P_X(\vx,\cdot)$ to the
target transition $P_Y(\vy,\cdot)$. Initialized from the stationary target
distribution independently of the fresh source trajectory, its recursive
application in \eqref{eq:intro-transition-goal} produces the exact target
trajectory law.

\textbf{Sample complexity.}
Theorem~\ref{thm:velocity-estimation} bounds the squared velocity error of
least-squares estimation, integrated over flow time, from two independent
stationary trajectories.
For a finite candidate class $\mathcal V$ containing the population velocity,
under suitable assumptions, error at most $\epsilon>0$
holds with probability at least $1-\delta$, for $\delta\in(0,1)$, if
\begin{equation}
 m\geq CK\max\{1,\epsilon^{-1}\}
 \log\left(\frac{6K|\mathcal V|}{\delta}\right),
 \label{eq:intro-sample-complexity}
\end{equation}
where $m$ counts transitions per trajectory, $K=\tau_{\rm mix}(\delta/(3m))$
is the observation spacing in Theorem~\ref{thm:velocity-estimation}, and $C$ depends only on the
assumption parameters. Compared with independent-sample analyses
\citep{zhou2025error,gaur2025continuous}, dependence replaces the sample size
$m$ by $m/K$, up to logarithmic factors, as in Markovian
regression \citep{nagaraj2020least}, streaming principal component analysis
\citep{kumar2023streaming}, and matrix concentration
\citep{neeman2024concentration}.

\textbf{Lower bound.}
Proposition~\ref{prop:mixing-lower-bound} adapts the lower-bound argument of
\citet[Appendix~C.1]{nagaraj2020least} to Gaussian conditional transitions.
Even with only two candidate velocities, achieving small squared velocity
error with high probability requires a number of transitions proportional
to the mixing time.

\textbf{Accuracy of generated states and trajectories.}
Theorem~\ref{thm:finite-rollout} in Appendix~\ref{app:stability-proofs} bounds
generation error under a coverage condition on generated state pairs.
For contractive target dynamics and coverage uniform over the horizon,
the state-marginal error bound has a decaying initialization term and an
estimation term bounded independently of trajectory length.
Proposition~\ref{prop:rollout} in the same appendix bounds trajectory Wasserstein distance with transport cost equal to the sum of squared state distances.

\textbf{Empirical evaluation.}
Section~\ref{sec:experiments} evaluates Gaussian and nonlinear transition
recovery and supervised EEG-to-image retrieval. For EEG, we freeze the
ENIGMA encoder and projection network pretrained on other subjects
\citep{kneeland2026enigma,enigma2026code}. Using 4,096 calibration images
per subject, we compare adapting only the subject-specific temporal map
with jointly training that map and a preceding width-64 flow conditioned
on the 32 previous EEG samples. This raises mean top-5 retrieval accuracy
from $43.87\%$ to $49.05\%$, an $11.82\%$ relative improvement.

\subsection{Related work}\label{sec:related-work}

\citet{kerrigan2024dynamic} and \citet{chemseddine2025conditional} construct flows
between conditional distributions. For Markov chains,
\citet{moulos2021bicausal} and \citet{oconnor2022optimal} optimize trajectory and
transition couplings through dynamic programming, while
\citet{oconnor2024estimation} estimate stationary transport plans and their
costs from observations. We use conditional flow matching to learn a
continuous transition map from one source and one target trajectory.
Conditioning on current source and target states lets us reuse this map to generate from a fresh source trajectory.

\citet{zhou2025error} and \citet{gaur2025continuous} give finite-sample error bounds
for neural flow matching trained on independent samples.
For sequential data, \citet{lim2026trajectory} express the empirical
flow-matching solution as a similarity-weighted average of observed
transition velocities and prove convergence as the trajectory grows.
We give finite-sample bounds for learning from dependent transitions:
for a finite candidate class, we bound the number of transitions needed
for a prescribed velocity error in terms of the transition mixing time,
and prove that a mixing-dependent sample threshold is necessary.
Appendix~\ref{app:related-work} surveys further related work.

\section{Preliminaries and notation}\label{sec:preliminaries}

% We introduce mixing to control training dependence and flow matching to transport transition distributions.

\paragraph{Notation.}
Here, a Markov process is discrete-time, possibly continuous state-space, and a trajectory is a finite
sequence of its states.
Vectors and vector-valued functions use bold lowercase letters, e.g.\  $\vx\in\R^d$ and the velocity field $\vv_t$. Matrices use
bold uppercase letters, e.g.\  $\ma\in\R^{d\times d}$. Scalars, indices, scalar-valued functions, and Markov kernels use ordinary italics. We use $\norm{\cdot}$ for the Euclidean norm for vectors and
$\norms{\ma}_{\rm F}=(\sum_{j,k}A_{jk}^2)^{1/2}$ for the Frobenius norm
of a matrix. For a random vector and $p\geq1$, $\norms{\vu}_{L^p}=(\E\norms{\vu}^p)^{1/p}$.
Distributions are probability measures; densities are not assumed to exist.
We write $\Law(\vu)$ for the law of $\vu$, so
$\Law(\vu)(A)=\Pp(\vu\in A)$ for every measurable set $A$, and
$g_\#\mu$ for the law of $g(\vu)$ when $\vu\sim\mu$.
The quadratic Wasserstein distance
$W_2(\mu,\nu)$ is the square root of the minimum $\E\norm{\vu-\vw}^2$ over joint
distributions with $\vu\sim\mu$ and $\vw\sim\nu$; such a joint distribution
is called a \emph{coupling}
\citep[Chapter~6]{villani2009optimal}. We also write
$\norm{\mu-\nu}_{\TV}=\sup_A|\mu(A)-\nu(A)|$ for total-variation distance.
Time is indexed by $i$, while
$t\in[0,1]$ denotes flow time within one transition. We write $\Unif(0,1)$
for the uniform distribution on $[0,1]$.

\paragraph{Markov processes and mixing.}
For a Markov kernel $P$, $P(\vx,\cdot)$ is the distribution of the next state
given the current state $\vx$. The first-order Markov property means
$\Law(\vx_{i+1}\mid\vx_{0:i})=P(\vx_i,\cdot)$; time homogeneity means
that $P$ does not depend on time index $i$. A distribution $\pi$ is stationary when
$\int P(\vx,A)\pi(\dd\vx)=\pi(A)$ for every set $A$. For a stationary Markov
process $(\vx_i)_{i\geq0}$, initialized with $\vx_0\sim\pi$, each
$\vx_i$ has distribution $\pi$. Using standard definitions of mixing~\citep{berbee1979random,yu1994rates,bradley2005basic}, let
\begin{equation}
 \beta_\ell=\E_{\vx\sim\pi}
 \norm{\Law(\vx_\ell\mid \vx_0=\vx)-\pi}_{\TV},
 \qquad
 \tau_{\rm mix}(\alpha)=\min\{\ell\geq1:\beta_\ell\leq\alpha\}.
 \label{eq:mixing-profile}
\end{equation}
We write $\tau_{\rm mix}=\tau_{\rm mix}(1/4)$ and call it the mixing time;
$\tau_{\rm mix}(\alpha)$ denotes the mixing time at any other tolerance
$\alpha\in(0,1)$. If $\beta_\ell\leq c_0e^{-\ell/\tau_0}$, then
$\tau_{\rm mix}(\alpha)\leq
\max\{1,\lceil\tau_0\log(c_0/\alpha)\rceil\}$.
% This profile and its use in coupling regularly spaced observations are standard
% \citep{berbee1979random,yu1994rates,bradley2005basic}.

\begin{remark}[Comparison with worst-case mixing]
Let $d_\ell^{\rm wc}=\sup_{\vx}\norm{\Law(\vx_\ell\mid \vx_0=\vx)-\pi}_{\TV}$ and
$\tau_{\rm wc}(\alpha)=\min\{\ell\geq1:d_\ell^{\rm wc}\leq\alpha\}$, the
standard worst-case total-variation mixing time
\citep[Chapter~4]{levin2017markov}. Since an average is at most a supremum,
$\beta_\ell\leq d_\ell^{\rm wc}$ and therefore
$\tau_{\rm mix}(\alpha)\leq\tau_{\rm wc}(\alpha)$. 

% The converse fails for
% $X_{i+1}=aX_i+\eta_{i+1}$ with $0<|a|<1$ and independent
% $\eta_i\sim N(0,1)$. Its stationary-average distance tends to zero
% (Appendix~\ref{app:gaussian}), but $d_\ell^{\rm wc}=1$ for each finite
% $\ell$, since the conditional mean $a^\ell X_0$ is unbounded while the
% conditional variance is fixed.
\end{remark}

\paragraph{Flow matching.}
Let $\mu_0$ and $\mu_1$ be distributions on $\R^d$ with finite second
moments. Draw $\vu_0\sim\mu_0$ and $\vu_1\sim\mu_1$ independently, and let
$\varphi:\R^d\to\R$ be any smooth compactly supported scalar function.
\begin{equation}
 \vz_t=(1-t)\vu_0+t\vu_1,
 \qquad
 \vv_t^\star(\vz)=\E[\vu_1-\vu_0\mid\vz_t=\vz],
 \qquad
 \frac{\dd}{\dd t}\E[\varphi(\vz_t)]
 =\E\left[\ip{\nabla\varphi(\vz_t)}{\vv_t^\star(\vz_t)}\right].
 \label{eq:generic-flow-matching}
\end{equation}
The interpolation $\vz_t$ joins the sampled endpoints by a straight line. The velocity $\vv_t^\star(\vz)$ averages the endpoint displacements
of paths passing through $\vz$ at time $t$, giving the local direction and
speed of transport. The function $\varphi$ tests the evolving interpolation law through
its expectation; smooth compact support means that $\varphi$ is infinitely differentiable and vanishes outside a bounded set.
Under the regularity in Assumption~\ref{ass:flow-regularity}, uniqueness for
this continuity equation identifies the interpolation law with the law
generated by the velocity's ordinary differential equation (ODE)
\citep[Proposition~8.1.7]{ambrosio2008gradient}. Specifically, for each
$\va\in\R^d$, define $\Phi_t(\va)$ by
\begin{equation*}
 \frac{\dd}{\dd t}\Phi_t(\va)=\vv_t^\star(\Phi_t(\va)),
 \qquad \Phi_0(\va)=\va.
\end{equation*}
Then $(\Phi_1)_\#\mu_0=\mu_1$, so $\Phi_1$ transports the initial
distribution to the target distribution \citep{lipman2023flow,liu2023flow}.
We apply this identity to conditional transition laws in
Proposition~\ref{prop:conditional-transport}.

\section{Problem and model}\label{sec:setup}

The model separates training trajectories from the fresh source trajectory used for generation.

\begin{model}[Markov process transport]
\label{model:markov-transport}
Fix $d,m,n\in\N$. Let $(\vx_i)_{i\geq0}$ and $(\vy_i)_{i\geq0}$ be
independent, time-homogeneous Markov processes on $\R^d$, with transition
kernels $P_X$ and $P_Y$. We observe one source trajectory $\vx_{0:m}$ and one
target trajectory $\vy_{0:m}$. At generation time, we receive a fresh source
trajectory $\vx'_{0:n}$ with transition kernel $P_X$, independent of both
training trajectories. Its initial distribution is arbitrary.
\end{model}
The finite-sample analysis additionally assumes stationary, mixing training data.
\begin{model}[Stationary training trajectories]
\label{model:stationary-training}
Instantiate Model~\ref{model:markov-transport} with stationary distributions
$\pi_X$ and $\pi_Y$, and initialize the training trajectories independently as
$\vx_0\sim\pi_X$ and $\vy_0\sim\pi_Y$. Write
$(\beta_\ell^X,\tau_X)$ and $(\beta_\ell^Y,\tau_Y)$ for their respective
mixing profiles and times in \eqref{eq:mixing-profile}.
Assume $\beta_\ell^X\to0$ and $\beta_\ell^Y\to0$ as $\ell\to\infty$.
The two processes need not have the same mixing profile or mixing time.
\end{model}

The generation objective specifies the conditional law of each output
transition.

\begin{problem}[Transport between Markov processes]
\label{prob:markov-transport}
Under Model~\ref{model:markov-transport}, given the two training trajectories
and a fresh source trajectory $\vx'_{0:n}$, learn a single family of maps
$(\Phi_{\vx,\vy})$ such that the recursion
\begin{equation}
 \widehat{\vy}_{i+1}
 =\Phi_{\vx'_i,\widehat{\vy}_i}(\vx'_{i+1})
 \label{eq:population-sampler}
\end{equation}
has transition kernel $P_Y$. If $P_Y$ has stationary distribution $\pi_Y$ and
$\widehat{\vy}_0\sim\pi_Y$ independently of $\vx'$, the desired output is a
stationary target trajectory. We call \eqref{eq:population-sampler} a \emph{rollout}.
\end{problem}

\paragraph{Conditional transport problem.}
Fix a time index $i$ and condition on the current states
$\vx_i=\vx$ and $\vy_i=\vy$. Here $\vx_i,\vy_i$ are random states,
whereas $\vx,\vy\in\R^d$ are their fixed values.
To construct the map $\Phi_{\vx,\vy}$ in \eqref{eq:population-sampler},
write $\vx^+=\vx_{i+1}$ and $\vy^+=\vy_{i+1}$. Their conditional laws are
\begin{equation}
 \Law(\vx^+\mid\vx_i=\vx,\vy_i=\vy)=P_X(\vx,\cdot),
 \qquad
 \Law(\vy^+\mid\vx_i=\vx,\vy_i=\vy)=P_Y(\vy,\cdot),
 \label{eq:conditional-endpoints}
\end{equation}
and they are conditionally independent.
Define
\begin{equation}
 \vz_t=(1-t)\vx^++t\vy^+,
 \qquad
 \vv_t^\star(\vz,\vx,\vy)
 =\E\left[\vy^+-\vx^+
 \mid \vz_t=\vz,\vx_i=\vx,\vy_i=\vy\right].
 \label{eq:population-velocity}
\end{equation}
Write $q_{t\mid\vx,\vy}=\Law(\vz_t\mid\vx_i=\vx,\vy_i=\vy)$.
For any velocity function $\vv$, let
$\Phi^{\vv}_{\vx,\vy,t}(\va)$ denote the solution at flow time $t$ of
\begin{equation}
 \frac{\dd\vz_t}{\dd t}=\vv_t(\vz_t,\vx,\vy),
 \qquad \vz_0=\va,
 \label{eq:conditional-ode}
\end{equation}
and write $\Phi^{\vv}_{\vx,\vy}=\Phi^{\vv}_{\vx,\vy,1}$. The required
one-transition identity is
\begin{equation}
 (\Phi_{\vx,\vy})_\#P_X(\vx,\cdot)=P_Y(\vy,\cdot).
 \label{eq:desired-kernel}
\end{equation}

The regression in \eqref{eq:population-velocity} applies standard flow matching
\citep{lipman2023flow,liu2023flow} conditional on $(\vx,\vy)$.
The ODE requires the following regularity.

\begin{assumption}[Flow regularity]\label{ass:flow-regularity}
At every current-state pair $(\vx,\vy)$ where a population flow is invoked,
the transition distributions have finite second moments. The population
velocity is measurable in $t$ and satisfies, for all $\vz,\vz'\in\R^d$ and
almost every $t\in[0,1]$,
\begin{equation}
 \norm{\vv_t^\star(\vz,\vx,\vy)-\vv_t^\star(\vz',\vx,\vy)}
 \leq L_t(\vx,\vy)\norm{\vz-\vz'},
 \qquad
 \norm{\vv_t^\star(\vz,\vx,\vy)}
 \leq G_t(\vx,\vy)(1+\norm{\vz}),
 \label{eq:flow-regularity}
\end{equation}
where $L_\cdot(\vx,\vy)$ and $G_\cdot(\vx,\vy)$ are nonnegative and integrable on
$[0,1]$. Whenever a learned flow is invoked, the learned velocity satisfies the
same conditions with its own nonnegative integrable functions $L_t$ and $G_t$.
\end{assumption}

These bounds give a unique solution on $[0,1]$ by
\citet[Lemma~8.1.4]{ambrosio2008gradient} and Gronwall's inequality, ensuring
the ODE uniqueness required in \citet[Definition~3.1]{liu2022marginal}.
\citet[Assumption~1(b)]{bansal2026straightness} similarly require the learned
velocity to be Lipschitz, with a constant uniform in flow time.
Appendix~\ref{app:gaussian} verifies the bounds for Gaussian transitions.
Algorithm~\ref{alg:transport} trains the velocity and generates the target
trajectory. Exact stationary generation uses $\widehat\mu_0=\pi_Y$; the
finite-sample bounds retain the error of approximate initialization.
\begin{algorithm}[b]
\caption{Conditional flow matching for transport between Markov processes}
\label{alg:transport}
\begin{minipage}{\dimexpr\linewidth-28pt\relax}
\begin{algorithmic}[1]
\Require Training trajectories $\vx_{0:m},\vy_{0:m}$; candidate velocity class $\mathcal V$; fresh source trajectory $\vx'_{0:n}$; initial distribution $\widehat\mu_0$.
\For{$i=0,\ldots,m-1$}\hfill\tikz[remember picture,overlay]\coordinate (train-start);
\State Set $\vr_i=\vy_{i+1}-\vx_{i+1}$ and $\vz_i(t)=(1-t)\vx_{i+1}+t\vy_{i+1}$ for $t\in[0,1]$.
\EndFor
\State Minimize $\frac1m\sum_{i=0}^{m-1}\int_0^1\norm{\vv_t(\vz_i(t),\vx_i,\vy_i)-\vr_i}^2\dd t$ over $\vv\in\mathcal V$ to obtain $\widehat{\vv}$.\hfill\tikz[remember picture,overlay]\coordinate (train-end);
\State Draw $\widehat{\vy}_0\sim\widehat\mu_0$ independently of $\vx'$.\hfill\tikz[remember picture,overlay]\coordinate (generation-start);
\For{$i=0,\ldots,n-1$}
\State Solve $\dot{\vz}_t=\widehat{\vv}_t(\vz_t,\vx'_i,\widehat{\vy}_i)$ on $[0,1]$ from $\vz_0=\vx'_{i+1}$.
\State Set $\widehat{\vy}_{i+1}=\vz_1$.
\EndFor
\State \Return $\widehat{\vy}_{0:n}$.\hfill\tikz[remember picture,overlay]\coordinate (generation-end);
\end{algorithmic}
\end{minipage}
\begin{tikzpicture}[remember picture,overlay]
\draw[decorate,decoration={brace,amplitude=4pt}]
 ($(train-start)+(5pt,7pt)$) -- ($(train-start |- train-end)+(5pt,-2pt)$)
 node[midway,xshift=15pt,rotate=90,font=\footnotesize]{Training};
\draw[decorate,decoration={brace,amplitude=4pt}]
 ($(train-start |- generation-start)+(5pt,7pt)$) -- ($(train-start |- generation-end)+(5pt,-2pt)$)
 node[midway,xshift=15pt,rotate=90,font=\footnotesize]{Generation};
\end{tikzpicture}
\end{algorithm}

\section{Main results}\label{sec:results}

Transport requires learning target transitions and preserving them under recursive application. We first establish
these properties for the population velocity. This reduces the statistical
problem to estimating that velocity from dependent observations, for which
we give an upper bound and a mixing-dependent lower bound.
We then bound generation error under suitable stability assumptions.

\subsection{From one transition to a trajectory}

For fixed current states $(\vx,\vy)$, applying the flow-matching identity \eqref{eq:generic-flow-matching} to
$P_X(\vx,\cdot)$ and $P_Y(\vy,\cdot)$ shows that the population ODE
follows the conditional interpolation law and reaches the target transition.

\begin{proposition}[Conditional transport]\label{prop:conditional-transport}
Under Model~\ref{model:markov-transport}, fix current-state values
$(\vx,\vy)$ at which Assumption~\ref{ass:flow-regularity} holds.
For the ODE map $\Phi^{\vv^\star}_{\vx,\vy,t}$ in
\eqref{eq:conditional-ode} and every $t\in[0,1]$,
\begin{equation}
 (\Phi^{\vv^\star}_{\vx,\vy,t})_\#P_X(\vx,\cdot)=q_{t\mid\vx,\vy}.
 \label{eq:conditional-marginal}
\end{equation}
In particular, $(\Phi^{\vv^\star}_{\vx,\vy})_\#P_X(\vx,\cdot)=P_Y(\vy,\cdot)$.
\end{proposition}

During recursive generation, the current target state depends on past
source states. We therefore need the transport identity to remain valid
after conditioning on this shared history. The source Markov property ensures that the next source state still has law
$P_X(\vx'_i,\cdot)$, to which we can apply the transport map.

\begin{lemma}[Conditional transitions during generation]\label{lem:current-state-sufficiency}
Under Model~\ref{model:markov-transport}, let
$\widehat{\vy}_0\sim\widehat\mu_0$ be independent of $\vx'$, and suppose
each $\widehat{\vy}_{i+1}$ is a fixed measurable function of
$(\vx'_{0:i+1},\widehat{\vy}_{0:i})$. Then, for every $i\geq0$,
\mbox{$\Law(\vx'_{i+1}\mid\vx'_{0:i},\widehat{\vy}_{0:i})=P_X(\vx'_i,\cdot)$}
almost surely.
If the updates follow \eqref{eq:population-sampler} with maps satisfying
\eqref{eq:desired-kernel} at every visited state pair, then, for every
$i\geq0$ and measurable $B\subseteq\R^d$, almost surely,
\begin{equation}
 \Pp\left(\widehat{\vy}_{i+1}\in B\mid
 \vx'_{0:i},\widehat{\vy}_{0:i}\right)=P_Y(\widehat{\vy}_i,B).
 \label{eq:current-state-sufficiency}
\end{equation}
\end{lemma}

The generated process thus has the target transition kernel conditional
on its full history. Its initial distribution determines its trajectory law; choosing $\pi_Y$ gives the stationary target law.
Appendix~\ref{app:population-proofs} gives the conditioning argument
and proves the following consequence.

\begin{proposition}[Exact trajectory generation]\label{prop:exact-trajectory}
Under Model~\ref{model:markov-transport}, suppose $P_Y$ has stationary
distribution $\pi_Y$, draw $\widehat{\vy}_0\sim\pi_Y$ independently of
$\vx'$, and generate by \eqref{eq:population-sampler}. If
\eqref{eq:desired-kernel} holds at every visited current-state pair, then
$\widehat{\vy}_{0:n}$ has the stationary target trajectory law for every $n\geq0$.
\end{proposition}

\subsection{Velocity estimation from two dependent trajectories}\label{sec:estimation}

Algorithm~\ref{alg:transport} estimates the population velocity by regressing next-state
displacements on interpolated states and the current-state pair.
The statistical difficulty is dependence between successive regression pairs from the two trajectories. For each observed transition and flow time $t\in[0,1]$, set
\begin{equation}
 \vz_i(t)=(1-t)\vx_{i+1}+t\vy_{i+1},\quad
 \vr_i=\vy_{i+1}-\vx_{i+1},\quad
 \vh_i(t)=(t,\vz_i(t),\vx_i,\vy_i).
 \label{eq:training-data}
\end{equation}

\paragraph{Regression noise.}
Let $(\vh,\vr)$ be obtained from a fresh stationary transition pair and an
independent flow time $t\sim\Unif(0,1)$ using \eqref{eq:training-data}.
The population velocity in \eqref{eq:population-velocity} satisfies
$\vv^\star(\vh)=\E[\vr\mid\vh]$. Define the regression noise by
$\boldsymbol\xi=\vv^\star(\vh)-\vr$ and
$\boldsymbol\xi_i(t)=\vv^\star(\vh_i(t))-\vr_i$.
The following assumption controls this regression noise.

\begin{assumption}[Conditional sub-Gaussian noise]\label{ass:conditional-noise}
There exists $\nu\geq0$ such that, almost surely, for every
$\va\in\R^d$ and $\theta\in\R$,
\begin{equation}
 \E\left[e^{\theta\ip{\va}{\boldsymbol\xi}}\mid\vh\right]
 \leq e^{\theta^2\nu^2\norm{\va}^2/2}.
 \label{eq:noise}
\end{equation}
By stationarity, the same bound holds for $(\vh_i(t),\vr_i)$ when $t$ is
drawn independently from $\Unif(0,1)$, for every $i$.
\end{assumption}

\begin{remark}
Assumption~\ref{ass:conditional-noise} holds with $\nu=B$ whenever
$\norms{\vr}\leq B$ almost surely. This includes source and target processes
on bounded state spaces. It also holds for transitions with uniformly
bounded additive noise and for Gaussian transition kernels with uniformly
bounded covariance; both allow nonlinear transition means.
Appendix~\ref{app:noise-examples} verifies these cases.
\end{remark}

We fit this conditional mean over a nonempty finite class $\mathcal V$ of measurable functions
$\vv:[0,1]\times\R^d\times\R^d\times\R^d\to\R^d$.
The finite-class assumption reflects finite-precision storage of neural-network parameters. For a fixed architecture with $p$
parameters, each taking at most $q$ values, we have $|\mathcal V|\leq q^p$
and $\log|\mathcal V|\leq p\log q$.
\citet[Section~3]{kumar2025dimensionfree} use the same interpretation.
We average the regression loss over flow time:
\begin{equation}
 \widehat L(\vv)=\frac1m\sum_{i=0}^{m-1}\int_0^1
 \norm{\vv(\vh_i(t))-\vr_i}^2\dd t,
 \qquad \vv\in\mathcal V.
 \label{eq:integrated-empirical-loss}
\end{equation}
Uniformly sampling transition indices and fresh independent flow times gives unbiased minibatch estimates, as in
\citet[Remark~2]{kumar2025dimensionfree}. Appendix~\ref{app:experiments}
specifies this procedure for our experiments. The statistical analysis
concerns empirical minimization; finite-budget optimization is a separate
source of error.
We minimize this loss and measure squared error relative to the population velocity:
\begin{equation}
 \widehat{\vv}\in\arg\min_{\vv\in\mathcal V}\widehat L(\vv),
 \qquad
 R(\vv,\vv^\star)=\E\norm{\vv(\vh)-\vv^\star(\vh)}^2,
 \label{eq:erm-risk}
\end{equation}
where $\vh$ is a fresh draw from the stationary training distribution.
To relate empirical fit to population error, we assume realizability and a fourth-moment bound on prediction errors.

\begin{assumption}[Realizability and fourth-moment bound]\label{ass:velocity-class}
The class is realizable, meaning $\vv^\star\in\mathcal V$. There is a
constant $C_{4\to2}\geq1$ such that, for every $\vv\in\mathcal V$,
$R(\vv,\vv^\star)<\infty$ and
\begin{equation}
 \left(\E\norm{\vv(\vh)-\vv^\star(\vh)}^4\right)^{1/4}
 \leq C_{4\to2}\sqrt{R(\vv,\vv^\star)}.
\label{eq:l4-l2}
\end{equation}
\end{assumption}

The moment comparison parallels \citet[Assumption~2]{kumar2025dimensionfree}
and the assumption in \citet[Section~14.2]{wainwright2019high},
here under the joint training distribution.
Appendix~\ref{app:gaussian} verifies realizability and $C_{4\to2}^4\leq8$
for the finite Gaussian velocity class.

To account for training dependence, note that the function $t\mapsto(\vh_i(t),\vr_i)$ is determined by the transition observation
$\vo_i=(\vx_i,\vx_{i+1},\vy_i,\vy_{i+1})$.
Consecutive observations $\vo_i$ and $\vo_{i+1}$ contain the same states
$\vx_{i+1}$ and $\vy_{i+1}$. Under Model~\ref{model:stationary-training},
this sequence is stationary and Markov by
Lemma~\ref{lem:paired-transition-mixing}. With stationary law $\pi_O$,
its observation mixing profile and mixing time are
\begin{equation}
 \beta_\ell=\E_{\vo_0\sim\pi_O}
 \norm{\Law(\vo_\ell\mid\vo_0)-\pi_O}_{\TV},
 \qquad
 \tau_{\rm mix}(\alpha)=\min\{\ell\geq1:\beta_\ell\leq\alpha\},
 \label{eq:observation-mixing}
\end{equation}
for integers $\ell\geq1$ and $\alpha\in(0,1)$, specializing
\eqref{eq:mixing-profile} to $(\vo_i)$.
Given $\vo_0$, the states $(\vx_\ell,\vy_\ell)$ evolve from
$(\vx_1,\vy_1)$ for $\ell-1$ steps. Independence bounds their joint
total-variation distance by the sum of the two chains' distances; drawing
the next states cannot increase that distance.
Lemma~\ref{lem:paired-transition-mixing} in
Appendix~\ref{app:estimation-proofs} therefore gives, for every
$\ell\geq2$ and $\alpha\in(0,1)$,
\begin{equation}
 \beta_\ell\leq\beta_{\ell-1}^X+\beta_{\ell-1}^Y,
 \qquad
 \tau_{\rm mix}(\alpha)
 \leq1+\max\{\tau_X(\alpha/2),\tau_Y(\alpha/2)\}.
 \label{eq:paired-transition-mixing}
\end{equation}
The slower process controls this upper bound; the mixing times may differ. We combine the noise, moment, and mixing assumptions. The resulting risk bound scales as $K/m$, up to class complexity
and logarithmic factors, where $K$ is the observation mixing time at
tolerance $\delta/(3m)$.

\begin{theorem}[Velocity estimation from dependent trajectories]\label{thm:velocity-estimation}
Under Model~\ref{model:stationary-training}, assume
Assumptions~\ref{ass:velocity-class} and~\ref{ass:conditional-noise}. Fix
$\delta\in(0,1)$ and set
$K=\tau_{\rm mix}(\delta/(3m))$ for the transition-observation
sequence $(\vo_i)$, as defined in \eqref{eq:observation-mixing}. Suppose
\small
\begin{equation}
 m\geq128C_{4\to2}^4K
 \log\left(\frac{6K|\mathcal V|}{\delta}\right).
 \label{eq:estimation-sample-size}
\end{equation}
\normalsize
Then the minimizer $\widehat{\vv}$ of \eqref{eq:integrated-empirical-loss}
satisfies, with probability at least $1-\delta$ over the two training trajectories,
\begin{equation}
 R(\widehat{\vv},\vv^\star)
 \leq128C_{4\to2}^4\nu^2
 \frac{K\log(6K|\mathcal V|/\delta)}m.
 \label{eq:estimation-bound}
\end{equation}
\end{theorem}

The same guarantee holds when each time integral is replaced by the average
of any fixed number $J\geq1$ of independent uniform flow-time samples per
transition, with probability also over these samples
(Corollary~\ref{cor:sampled-flow-times}). The sample size $m$ counts observed
transitions in both cases.

To see how these ingredients enter, put
$\vg_i(\vv,t)=\vv(\vh_i(t))-\vv^\star(\vh_i(t))$ and decompose the excess loss:
\begin{equation}
 \widehat L(\vv)-\widehat L(\vv^\star)
 =\frac1m\sum_i\int_0^1
 \left(\norm{\vg_i(\vv,t)}^2
       +2\ip{\vg_i(\vv,t)}{\boldsymbol\xi_i(t)}\right)\dd t.
 \label{eq:empirical-excess-decomposition}
\end{equation}
At the empirical minimizer, the left side is nonpositive by realizability.
Conditional sub-Gaussianity and the fourth-moment assumption bound the
exponential moment of the negative excess loss at a random flow time.
Jensen's inequality preserves this bound when we average over flow time.
To handle dependence, we split the indices into subsequences
$a,a+K,a+2K,\ldots$ for
$a\in\{0,\ldots,K-1\}$. We couple each subsequence
to independent observations and bound the probability of any mismatch.
H\"older's inequality combines the exponential bounds for these subsequences
without assuming independence between them. In Appendix~\ref{app:estimation-proofs}, we
give the full proof and derive the sample complexity in \eqref{eq:intro-sample-complexity}
by combining \eqref{eq:estimation-sample-size} and \eqref{eq:estimation-bound}.

To establish the necessity of mixing-dependent sample complexity, we construct two models that no estimator can reliably distinguish from a short
trajectory. We do so by
adapting the unvisited-state argument of
\citet[Appendix~C.1]{nagaraj2020least}: a slowly mixing chain
may avoid the region where two candidate velocities differ, leaving their
models indistinguishable on the observed trajectories.

\begin{proposition}[Mixing-time lower bound]\label{prop:mixing-lower-bound}
For every integer $\tau\geq1$ and $\epsilon>0$, there exist two instances of
Model~\ref{model:stationary-training} which satisfy
Assumptions~\ref{ass:conditional-noise}, \ref{ass:flow-regularity}, and
\ref{ass:velocity-class} with universal constants $\nu,C_{4\to2}$,
share a class $\mathcal V$ of two velocities, and both have observation
mixing time $\Theta(\tau)$, such that
for any $\delta\in(0,1/(2\pi))$ and any estimator
$\widetilde{\vv}\in\mathcal V$ which achieves
$R(\widetilde{\vv},\vv^\star)<\epsilon$ with probability at least $1-\delta$
under both instances using $m$ observed transitions from each training
trajectory,
\begin{equation}
 m\geq\frac13\tau_{\rm mix}\log\frac1{2\pi\delta}.
 \label{eq:mixing-lower-sample}
\end{equation}
Here $\tau_{\rm mix}$ denotes either instance's observation mixing time
at tolerance $1/4$.
\end{proposition}

\begin{remark}[Comparing the upper and lower bounds]
The lower bound uses $\tau_{\rm mix}=\tau_{\rm mix}(1/4)$, whereas the
upper bound uses $K=\tau_{\rm mix}(\delta/(3m))$. We compare these rates under geometric mixing, i.e,
$\beta_\ell\leq c_0e^{-c_1\ell/\tau_{\rm mix}}$ for all $\ell\geq1$
and universal $c_0\geq1,c_1>0$. Then
$K=O(\tau_{\rm mix}\log(3m/\delta))$, so the bounds agree in mixing-time dependence up to
logarithmic factors. Appendix~\ref{app:mixing-comparison} derives this
comparison, verifies the exponential profile for our Gaussian construction,
and gives a variant under finite worst-case mixing time using total-variation
contraction \citep{levin2017markov}.
\end{remark}

% Finally, we connect estimation with generation by controlling velocity error at the
% state pairs visited during rollout. Under suitable regularity, coverage, and
% contraction assumptions stated in Theorem~\ref{thm:finite-rollout}
% (Appendix~\ref{app:stability-proofs}),
% the sample size in \eqref{eq:intro-sample-complexity} ensures that the
% exact-flow rollout satisfies
% $W_2(\Law(\widehat{\vy}_n),\pi_Y)
% \leq\rho^nW_2(\widehat\mu_0,\pi_Y)
% +e^\Lambda\sqrt{C_{\rm roll}\epsilon}/(1-\rho)$
% with probability at least $1-\delta$.
% Here $0\leq\rho<1$ is the target contraction factor, $C_{\rm roll}$ is the
% coverage factor, and $\Lambda$ is the integrated Lipschitz bound of the
% learned velocity. Appendix~\ref{app:stability-proofs} provides the detailed assumptions, statement and proofs of our result.

\section{Experiments}\label{sec:experiments}

We test Gaussian and nonlinear transition recovery and EEG image retrieval.
Lemma~\ref{lem:current-state-sufficiency} identifies the source law as
$P_X(\vx_i,\cdot)$, so conditioning also on $\vx_i$ gives the exact trajectory
guarantee of Proposition~\ref{prop:exact-trajectory}. We compare
our flow, an ablation omitting $\vx_i$, an unconditional marginal flow,
and an oracle. Marginal flow learns a single map $T$ between state
distributions and outputs $T(\vx_{i+1})$ without either context input.
Appendices~\ref{app:experiments} and~\ref{app:enigma-cfm} specify the protocols. Figure~\ref{fig:synthetic}(a--b) varies training size at $d=10$.
We measure conditional next-state error by sliced $W_2$
\citep{rowland2019orthogonal}, defined in Appendix~\ref{app:experiments}.
Both oracles transport the supplied source next state using the true transition laws. Their positive error reflects finite-sample evaluation.
% \red{see \eqref{eq:gaussian-oracle-benchmark}.}

\begin{figure}[!t]
 \centering
 \begin{subfigure}[t]{.32\linewidth}
  \centering
  \includegraphics[width=\linewidth]{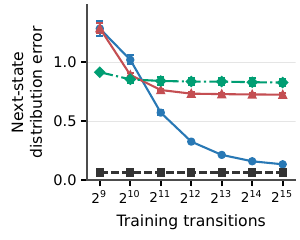}
  \caption{Gaussian transitions}
  \label{fig:synthetic-gaussian}
 \end{subfigure}\hfill
 \begin{subfigure}[t]{.32\linewidth}
  \centering
  \includegraphics[width=\linewidth]{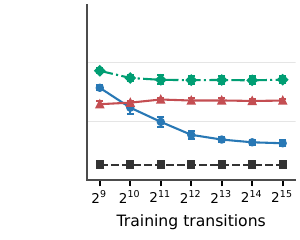}
  \caption{Nonlinear mixture transitions}
  \label{fig:synthetic-nonlinear}
 \end{subfigure}\hfill
 \begin{subfigure}[t]{.32\linewidth}
  \centering
  \includegraphics[width=\linewidth]{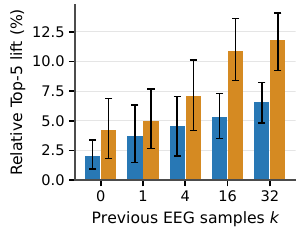}
  \caption{EEG image retrieval}
  \label{fig:enigma-main}
 \end{subfigure}
 \par\includegraphics[width=.66\linewidth]{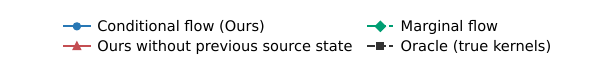}\hfill
 \includegraphics[width=.32\linewidth]{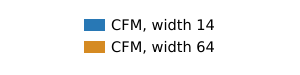}
 \caption{Synthetic transition recovery and EEG image retrieval.
 (a--b) Next-state distribution error for $d=10$ (mean $\pm$ standard error
 over three seeds; lower is better). The ablation omits $\vx_i$; marginal
 flow omits both context states; the oracle uses true kernels.
 (c) Relative Top-5 retrieval accuracy lift over ENIGMA
 with only $\mw_s$ adapted, versus EEG context length $k$; both flow
 networks precede $\mw_s$ and have two hidden layers: width 64 and width 14.
 Accuracies average three seeds per subject, then ten subjects; error bars show
 95\% subject-bootstrap percentile intervals.
 }
 \label{fig:synthetic}
 \label{fig:nonlinear-scaling}
\end{figure}

\paragraph{Gaussian transitions.}

We simulate independent processes in $\R^{10}$ with Gaussian transitions
$P_X(\vx,\cdot)=\Normal(\ma_X\vx,\mq_X)$ and
$P_Y(\vy,\cdot)=\Normal(\ma_Y\vy,\mq_Y)$, where $\ma_X,\ma_Y$ are
transition matrices and $\mq_X,\mq_Y$ are positive-definite noise covariances.
Source and target transition eigenvalue magnitudes lie in $[0.72,0.94]$
and $[0.15,0.58]$, respectively; all four matrices have independent random
orthogonal eigenbases.
Each process supplies $m\in\{2^9,2^{10},\ldots,2^{15}\}$ training transitions.
Figure~\ref{fig:synthetic}(a) shows that our error falls from $1.286$ at
$m=512$ to $0.133$ at $m=32768$, versus $0.725$ for the ablation and
$0.827$ for marginal flow. Conditioning improves recovery with more data.

\paragraph{Nonlinear mixture transitions.}

For $d=10$, let $\mR_X,\mR_Y$ be independent random orthogonal matrices.
Define $\vf(\vx)=0.85\mR_X\tanh(\mR_X^\top\vx)$,
$\vg(\vy)=\mR_Y[0.55\tanh(\mR_Y^\top\vy)+0.2\sin(2\mR_Y^\top\vy)]$,
and $\vu=d^{-1/2}\mR_Y(1,\ldots,1)^\top$, where $\tanh$ and $\sin$ act
coordinatewise. Each map applies scalar nonlinearities in rotated coordinates, then transforms back, coupling the observed coordinates. We use $P_X(\vx,\cdot)=\Normal(\vf(\vx),0.6^2\Id)$ and
$P_Y(\vy,\cdot)=\tfrac12\Normal(\vg(\vy)-0.8\sqrt d\,\vu,0.25^2\Id)
+\tfrac12\Normal(\vg(\vy)+0.8\sqrt d\,\vu,0.25^2\Id)$.
Each target noise vector uses one mixture sign, giving two modes along $\vu$.
We use the same training-size sweep as in panel (a).
In Figure~\ref{fig:synthetic}(b), conditioning improves with more data: at $m=32768$, our error is
$0.312$ versus $0.675$ for the ablation, $0.852$ for marginal flow, and
$0.131$ for the oracle. The ablation performs better at the smallest sample
sizes, while our gains taper at larger sizes under the fixed training budget.
Appendix~\ref{app:highdim-stress} reports the corresponding $d=100$ results.

\paragraph{Supervised EEG adaptation.}

In THINGS-EEG2 \citep{gifford2022things}, participants view images while
electroencephalography (EEG) records electrical activity from 63 scalp
channels. Each recording spans one second after an image appears. The task is to learn a decoder which predicts the viewed image among the 200 candidate images. Responses to the same image vary across people, motivating adaptation of a decoder for new subjects. We test whether a flow conditioned on preceding EEG measurements improves adaptation by transforming the recording before decoding.

For subject $s$ viewing image $u$, let
$\ma=\mE_s(u)=(\va_1,\ldots,\va_{250})\in\R^{63\times250}$.
Our ENIGMA baseline \citep{kneeland2026enigma,enigma2026code} predicts
$D(\mw_s(\ma))$, adapting the subject-specific temporal map $\mw_s$
while freezing the encoder and projection network $D$ pretrained on other
subjects. We predict $D(\mw_s(F_{\vtheta}(\ma)))$, jointly training the
flow $F_{\vtheta}$ and $\mw_s$. Both methods share 4,096
calibration and 1,000 validation images per subject; evaluation averages
80 EEG repetitions per image.

Training pairs $\ma$ with each reference subject's recording
$\mb_q=\mE_q(u)$ of the same image, with columns $\vb_{qj}$.
For $0\leq k\leq32$, the history
$\vh_k(\ma,j)=(\va_{j-1},\ldots,\va_{j-k},\vzero,\ldots,\vzero)
\in\R^{63\cdot32}$ contains the previous $k$ samples, with missing
samples and unused slots set to zero; define $\vh_k(\mb_q,j)$ similarly.
This keeps network size fixed as $k$ varies.
With $t\sim\Unif(0,1)$ and $\vz_{qj}=(1-t)\va_j+t\vb_{qj}$,
we combine ENIGMA's supervised objective with the mean flow-matching loss
\[
 \norm{\vv_{\vtheta}\left(t,\vz_{qj};
 \vh_k(\ma,j),\vh_k(\mb_q,j)\right)-(\vb_{qj}-\va_j)}^2.
\]
At inference, we generate $\widehat{\mb}=F_{\vtheta}(\ma)$ sequentially
by integrating from each $\va_j$, using previously generated target
history $\vh_k(\widehat{\mb},j)$ and holding both histories fixed during
each solve. Only the source recording is required.
Remaining details are in Appendix~\ref{app:enigma-cfm}.

Figure~\ref{fig:enigma-main} compares widths 64 and 14 at
$k\in\{0,1,4,16,32\}$. Top-5 accuracy is the frequency of ranking the viewed image among the top five candidates. At $k=32$, our width-64 model
achieves $49.05\%$ accuracy versus $43.87\%$ for the baseline, an
$11.82\%$ relative lift (95\% subject-bootstrap interval
$[9.25,14.06]\%$), up from $4.18\%$ lift at $k=0$.
The width-14 model similarly reaches $6.53\%$ lift at $k=32$.
Both widths benefit from temporal context and improve on the baseline
for all ten subjects after averaging seeds.

\section{Conclusion}

We prove that conditional flow matching transforms a fresh source trajectory
into the exact target trajectory distribution at the population level.
For a finite candidate class, we prove a sample-complexity upper bound
proportional to mixing time up to logarithmic factors, a lower bound showing
that this dependence is unavoidable, and a finite-sample rollout bound.
Our synthetic experiments show the value of conditioning on the previous
source state, and our EEG study shows improved adaptation between subjects.

Extending the analysis beyond finite classes, identifying when accuracy on
observed transitions carries over to generated trajectories, and handling
longer memory remain open directions. Generation can be slow because it
requires sequential ODE solves along the trajectory. Reducing this latency
is another interesting direction, potentially through parallel Picard
iteration \citep{shih2023parallel} or training flows to generate several
future states jointly.

\section{Acknowledgements}
SK gratefully acknowledges funding support from the Amazon AI PhD Fellowship. PS gratefully acknowledges NSF grants 2217069 and CCF-2505865.

\clearpage
\section*{Ethics statement}
Our EEG experiments reuse the publicly released THINGS-EEG2 dataset
\citep{gifford2022things}; we collect no new human-subject data. The dataset
is distributed under a CC BY 4.0 license. EEG-based image retrieval raises
privacy concerns because it seeks to infer viewed stimuli from neural
recordings. Our evaluation is limited to ranking a fixed set of candidate
images in a controlled experimental setting. Any use beyond this setting
requires separate validation and safeguards for participant consent and
neural-data privacy.

\section*{Reproducibility statement}
Appendix~\ref{app:proofs} contains the proofs, and Appendices~\ref{app:experiments} and~\ref{app:enigma-cfm} describe the experimental procedures.

\section*{AI use statement}
% Adapted from https://iclr.cc/Conferences/2027/AIPolicyForAuthors .
In this work, we used generative AI tools to explore ideas for the rollout
proof and the lower bound, assist with writing and revising proofs, and
support experiment implementation.
Additionally, we used these tools for manuscript editing, literature searches
and citation checking, code development, and figure preparation.
The authors reviewed the AI-assisted work by checking the proofs, testing
the code, and verifying the references.
The authors take responsibility for the final manuscript, mathematical
claims, code, figures, and experimental results, including content produced
with the assistance of generative AI.

% Keep references and supplementary material separate from the main text.
\clearpage
\bibliographystyle{plainnat}
\bibliography{refs}

\clearpage
\appendix
\section*{Appendix contents}
\begingroup
\newcommand{\appentry}[2]{%
  \noindent\hyperref[#1]{\ref*{#1}\quad #2}\dotfill\pageref{#1}\par}
\newcommand{\appsubentry}[2]{%
  \noindent\hspace*{1em}\hyperref[#1]{\ref*{#1}\quad #2}\dotfill\pageref{#1}\par}
\appentry{app:proofs}{Proofs}
\appsubentry{app:population-proofs}{Conditional transport and trajectory generation}
\appsubentry{app:estimation-proofs}{Dependent velocity estimation}
\appsubentry{app:lower-bound}{The sample-size lower bound}
\appsubentry{app:mixing-comparison}{Comparison of mixing tolerances}
\appsubentry{app:stability-proofs}{Accuracy of generated states and trajectories}
\smallskip
\appentry{app:gaussian}{Gaussian transitions, estimation, and rollout}
\appentry{app:noise-examples}{Examples for the conditional noise assumption}
\smallskip
\appentry{app:experiments}{Experimental details}
\appsubentry{app:gaussian-experiments}{Gaussian transition experiments}
\appsubentry{app:nonlinear}{Nonlinear mixture benchmark}
\appsubentry{app:highdim-stress}{Results at dimension 100}
\smallskip
\appentry{app:enigma-cfm}{Ten-subject EEG study}
\appsubentry{app:eeg-protocol}{Protocol}
\appsubentry{app:eeg-adapters}{Adapter configurations}
\appsubentry{app:eeg-throughput}{Additional results}
\appsubentry{app:eeg-post-ws}{Flow after the temporal map}
\smallskip
\appentry{app:related-work}{Further comparisons with prior work}
\appsubentry{app:related-transport}{Sequential transport and generation}
\appsubentry{app:related-statistics}{Statistical error and dependence}
\appsubentry{app:related-rollout}{Distribution mismatch and rollout stability}
\endgroup
\clearpage

\section{Proofs}\label{app:proofs}

\subsection{Conditional transport and trajectory generation}\label{app:population-proofs}

Proposition~\ref{prop:conditional-transport} identifies the law of a single
population flow. Lemma~\ref{lem:current-state-sufficiency} verifies that its
source input retains the required conditional law during generation; together
they give Proposition~\ref{prop:exact-trajectory}.

\begin{proof}[Proof of Proposition~\ref{prop:conditional-transport}]
We show that the interpolation law solves the continuity equation, then use
uniqueness to identify it with the ODE flow law.
Condition on $\vx_i=\vx$ and $\vy_i=\vy$. Throughout this proof,
all laws and expectations are under this conditioning, which we suppress
from the notation. Let
$\vz_t=(1-t)\vx^++t\vy^+$, $q_t=\Law(\vz_t)$, and
$\vr=\vy^+-\vx^+$. For every smooth
compactly supported $\varphi:\R^d\to\R$, finite first moments and conditioning on
$\vz_t$ give
\begin{align*}
 \frac{\dd}{\dd t}\int\varphi(\vz)q_t(\dd\vz)
 &=\E\ip{\nabla\varphi(\vz_t)}{\vr}
 =\E\ip{\nabla\varphi(\vz_t)}{\E[\vr\mid\vz_t]}\\
 &=\int\ip{\nabla\varphi(\vz)}{\vv_t^\star(\vz)}q_t(\dd\vz).
\end{align*}
Thus $q_t$ solves the continuity equation with velocity $\vv_t^\star$ and
initial law $P_X(\vx,\cdot)$. To identify this solution, observe that the
ODE flow in \eqref{eq:conditional-ode} has the same velocity and initial law.
Uniqueness under Assumption~\ref{ass:flow-regularity}
\citep[Proposition~8.1.7]{ambrosio2008gradient} therefore gives
$\Law(\Phi^{\vv^\star}_{\vx,\vy,t}(\vx^+))
=q_{t\mid\vx,\vy}$. At $t=1$, this is $P_Y(\vy,\cdot)$.
\end{proof}

\begin{proof}[Proof of Lemma~\ref{lem:current-state-sufficiency}]
We first identify the source law conditional on the generated history,
then apply the transport identity to obtain the target transition.
Independence of
$\widehat{\vy}_0$ and $\vx'$, followed by the source Markov property, gives
\begin{equation*}
 \Law\left(\vx'_{i+1}\mid\vx'_{0:i},\widehat{\vy}_0\right)
 =\Law\left(\vx'_{i+1}\mid\vx'_{0:i}\right)
 =P_X(\vx'_i,\cdot).
\end{equation*}
By recursion, $\widehat{\vy}_{0:i}$ is a deterministic function of
$(\vx'_{0:i},\widehat{\vy}_0)$. Since the generated history includes
$\widehat{\vy}_0$, these two histories generate the same sigma-field:
\begin{equation*}
 \sigma(\vx'_{0:i},\widehat{\vy}_{0:i})
 =\sigma(\vx'_{0:i},\widehat{\vy}_0).
\end{equation*}
This proves the first assertion. For updates of the form
\eqref{eq:population-sampler}, the transport identity
\eqref{eq:desired-kernel} at $(\vx'_i,\widehat{\vy}_i)$ now gives, for every
measurable $B\subseteq\R^d$,
\begin{align*}
 \Pp(\widehat{\vy}_{i+1}\in B\mid\vx'_{0:i},\widehat{\vy}_{0:i})
 &=\int\mathbf1\{\Phi_{\vx'_i,\widehat{\vy}_i}(\vu)\in B\}
 P_X(\vx'_i,\dd\vu)\\
 &=P_Y(\widehat{\vy}_i,B).
\end{align*}

\end{proof}

\begin{proof}[Proof of Proposition~\ref{prop:exact-trajectory}]
By Lemma~\ref{lem:current-state-sufficiency} and the tower property,
$\Law(\widehat{\vy}_{i+1}\mid\widehat{\vy}_{0:i})
=P_Y(\widehat{\vy}_i,\cdot)$. Starting from
$\widehat{\vy}_0\sim\pi_Y$, iteration gives the joint law
\begin{equation*}
 \pi_Y(\dd\vy_0)\prod_{j=0}^{n-1}P_Y(\vy_j,\dd\vy_{j+1}),
\end{equation*}
which is the stationary target trajectory distribution.
\end{proof}

\subsection{Dependent velocity estimation}\label{app:estimation-proofs}

Lemma~\ref{lem:paired-transition-mixing} bounds the mixing of the transition
observations. Lemma~\ref{lem:coupling} then couples regularly spaced
observations to independent copies. Lemma~\ref{lem:averaged-excess} bounds
the exponential moment of each excess loss, and Lemma~\ref{lem:excess-risk}
combines these bounds to control the risk of an empirical minimizer.

\begin{lemma}[Mixing of transition observations]\label{lem:paired-transition-mixing}
Under Model~\ref{model:stationary-training}, let $(\vo_i)_{i\geq0}$ be the
transition observations in \eqref{eq:observation-mixing}. Then $(\vo_i)$ is a
stationary, time-homogeneous Markov process. For every integer $\ell\geq2$ and
$\alpha\in(0,1)$,
\begin{equation*}
 \beta_\ell\leq\beta_{\ell-1}^X+\beta_{\ell-1}^Y,
 \qquad
 \tau_{\rm mix}(\alpha)
 \leq1+\max\{\tau_X(\alpha/2),\tau_Y(\alpha/2)\}.
\end{equation*}
\end{lemma}

\begin{proof}
We express each observation as a random update of the current-state pair.
Total-variation contraction then reduces its mixing bound to those of the
two independent chains.

For
probability laws $\mu,\nu$ and a Markov kernel $Q$, write
$(\mu Q)(A)=\int Q(u,A)\mu(\dd u)$. Since $Q(u,A)\in[0,1]$,
\begin{align*}
 |(\mu Q)(A)-(\nu Q)(A)|
 &=\left|\int_0^1\left(\mu\{u:Q(u,A)>s\}
             -\nu\{u:Q(u,A)>s\}\right)\dd s\right|\\
 &\leq\norm{\mu-\nu}_{\TV}.
\end{align*}
Taking the supremum over $A$ proves
\begin{equation}
 \norm{\mu Q-\nu Q}_{\TV}\leq\norm{\mu-\nu}_{\TV}.
 \label{eq:tv-kernel-contraction}
\end{equation}
This contraction also applies to a kernel that appends an independent draw.
Combining it with the triangle inequality, using intermediate law
$\nu_X\otimes\mu_Y$, gives
\begin{align*}
 \norm{\mu_X\otimes\mu_Y-\nu_X\otimes\nu_Y}_{\TV}
 &\leq\norm{\mu_X\otimes\mu_Y-\nu_X\otimes\mu_Y}_{\TV}
      +\norm{\nu_X\otimes\mu_Y-\nu_X\otimes\nu_Y}_{\TV}\\
 &\leq\norm{\mu_X-\nu_X}_{\TV}+\norm{\mu_Y-\nu_Y}_{\TV}.
\end{align*}

To apply these inequalities, define the observation kernel as follows.
Let $Q((\vx,\vy),\cdot)$ be the law of
$(\vx,\vx^+,\vy,\vy^+)$ obtained by drawing
$\vx^+\sim P_X(\vx,\cdot)$ and $\vy^+\sim P_Y(\vy,\cdot)$ independently.
The Markov property and independence of the chains give
\begin{equation*}
 \Law(\vo_{i+1}\mid\vo_0,\ldots,\vo_i)
 =Q((\vx_{i+1},\vy_{i+1}),\cdot).
\end{equation*}
Thus $(\vo_i)$ is time-homogeneous and Markov. Stationarity of the two
trajectories gives its stationary law
$\pi_O=(\pi_X\otimes\pi_Y)Q$.

Conditioning on $\vo_0$ fixes
$\vx_1$ and $\vy_1$. For $\ell\geq2$,
the two chains then evolve independently for $\ell-1$ steps before the
operation $Q$ constructs $\vo_\ell$. Therefore, almost surely,
\begin{equation*}
 \Law(\vo_\ell\mid\vo_0)
 =\left(P_X^{\ell-1}(\vx_1,\cdot)
       \otimes P_Y^{\ell-1}(\vy_1,\cdot)\right)Q.
\end{equation*}
Applying the two total-variation inequalities above yields
\begin{align*}
 \norm{\Law(\vo_\ell\mid\vo_0)-\pi_O}_{\TV}
 &\leq\norm{P_X^{\ell-1}(\vx_1,\cdot)
       \otimes P_Y^{\ell-1}(\vy_1,\cdot)-\pi_X\otimes\pi_Y}_{\TV}\\
 &\leq\norm{P_X^{\ell-1}(\vx_1,\cdot)-\pi_X}_{\TV}
       +\norm{P_Y^{\ell-1}(\vy_1,\cdot)-\pi_Y}_{\TV}.
\end{align*}
Averaging over $\vo_0$, whose next-state marginals are $\pi_X$ and $\pi_Y$, gives
\begin{equation*}
 \beta_\ell
 =\E\norm{\Law(\vo_\ell\mid\vo_0)-\pi_O}_{\TV}
 \leq\beta_{\ell-1}^X+\beta_{\ell-1}^Y.
\end{equation*}

For the mixing-time bound, \eqref{eq:tv-kernel-contraction} also implies
monotonicity: for a Markov kernel $P$ with stationary law $\pi$,
\begin{equation*}
 \norm{P^{k+1}(u,\cdot)-\pi}_{\TV}
 =\norm{P^k(u,\cdot)P-\pi P}_{\TV}
 \leq\norm{P^k(u,\cdot)-\pi}_{\TV}.
\end{equation*}
Hence each stationary-average mixing profile is nonincreasing. For
$k=\max\{\tau_X(\alpha/2),\tau_Y(\alpha/2)\}$, we obtain
$\beta_{k+1}\leq\beta_k^X+\beta_k^Y\leq\alpha$, proving the mixing-time bound.
\end{proof}

For an integer spacing $1\leq K\leq m$, write
$I_a=\{a+jK:j\geq0,\ a+jK<m\}$ for $0\leq a<K$.
The next lemma supplies independent observations within each $I_a$ while
preserving the joint law of the original data.

\begin{lemma}[Residue-class coupling]\label{lem:coupling}
Let $(\vo_i)$ be a stationary, time-homogeneous Markov process on a Euclidean space with
stationary law $\pi_O$ and mixing profile $\beta_\ell$ defined by
\eqref{eq:observation-mixing}. For every integer $1\leq K\leq m$, there are copies
$(\overline{\vo}_i)_{i=0}^{m-1}$ jointly coupled with $(\vo_i)$ such that
$(\overline{\vo}_i)_{i\in I_a}$ are iid with law $\pi_O$ for each $a$, and
\begin{equation}
 \Pp(\overline{\vo}_i\neq\vo_i\text{ for at least one }i)\leq m\beta_K.
 \label{eq:coupling-error}
\end{equation}
The coupling preserves the joint law of $(\vo_i)$; the copies in different
residue classes may be dependent.
\end{lemma}

\begin{proof}
We construct independent copies sequentially within each residue class,
then bound the probability of any mismatch by a union bound. The construction
uses maximal coupling \citep[Corollary~4.2.5]{berbee1979random}: given
$\vu\sim\mu$, auxiliary
randomness produces $\overline{\vu}\sim\nu$ with
$\Pp(\overline{\vu}\ne\vu)=\norm{\mu-\nu}_{\TV}$.

Write $P_O$ for the transition kernel. Fix a residue class
$I_a=\{a,a+K,\ldots\}\cap\{0,\ldots,m-1\}$, with indices
$i_1<\cdots<i_r$, and keep $\overline{\vo}_{i_1}=\vo_{i_1}\sim\pi_O$.
For $j\geq2$, suppose the earlier copies have been constructed from the
earlier observations in this class and independent auxiliary randomness.
The Markov property gives
\begin{equation*}
 \Law\left(\vo_{i_j}\mid
 \vo_{i_1},\ldots,\vo_{i_{j-1}},
 \overline{\vo}_{i_1},\ldots,\overline{\vo}_{i_{j-1}}\right)
 =P_O^K(\vo_{i_{j-1}},\cdot).
\end{equation*}
Apply maximal coupling conditionally to this law and $\pi_O$, using fresh
auxiliary randomness to construct $\overline{\vo}_{i_j}$ from
$\vo_{i_j}$. The conditional law of the new copy is $\pi_O$, so it is
independent of all earlier copies in this class. Its mismatch probability is
\begin{align*}
 \Pp(\overline{\vo}_{i_j}\ne\vo_{i_j})
 &=\E\norm{P_O^K(\vo_{i_{j-1}},\cdot)-\pi_O}_{\TV}\\
 &=\E_{\vo\sim\pi_O}\norm{P_O^K(\vo,\cdot)-\pi_O}_{\TV}
 =\beta_K.
\end{align*}
Repeat this construction separately for each class, using independent
auxiliary randomness. The original observations remain unchanged, and a
union bound yields
\begin{equation*}
 \Pp(\exists i:\overline{\vo}_i\ne\vo_i)
 \leq\sum_{a=0}^{K-1}(|I_a|-1)\beta_K
 =(m-K)\beta_K\leq m\beta_K.
\end{equation*}
\end{proof}

For $\vv\in\mathcal V$, recall the prediction error
$\vg_i(\vv,t)=\vv(\vh_i(t))-\vv^\star(\vh_i(t))$ from
\eqref{eq:empirical-excess-decomposition}. Define the excess loss and its
average over flow time by
\begin{equation}
 \begin{aligned}
 \excessloss_i(\vv,t)
 &=\norm{\vv(\vh_i(t))-\vr_i}^2
   -\norm{\vv^\star(\vh_i(t))-\vr_i}^2\\
 &=\norm{\vg_i(\vv,t)}^2
   +2\ip{\vg_i(\vv,t)}{\boldsymbol\xi_i(t)},
 &\avgexcess_i(\vv)&=\int_0^1\excessloss_i(\vv,t)\dd t.
 \end{aligned}
 \label{eq:averaged-excess-loss}
\end{equation}
Thus $m^{-1}\sum_i\avgexcess_i(\vv)=\widehat L(\vv)-\widehat L(\vv^\star)$.
The noise and moment assumptions ensure that the integrated losses are finite
almost surely: $\E\norm{\boldsymbol\xi}^2\leq d\nu^2$ and every candidate has finite
prediction risk. All time evaluations of a transition share its endpoints;
we keep them together when bounding the excess loss.

\begin{lemma}[Exponential bounds for excess loss]\label{lem:averaged-excess}
Under Assumptions~\ref{ass:conditional-noise} and~\ref{ass:velocity-class},
for a stationary transition and $\nu>0$, define
\begin{equation}
 \lambda=\frac1{4\nu^2},\qquad
 \riskdecay(u)=\frac1{4C_{4\to2}^4}
 \left(1-e^{-u/(16\nu^2)}\right),\quad u\geq0.
 \label{eq:excess-decay}
\end{equation}
Let $t\sim\Unif(0,1)$ be independent of the transition. For every
$\vv\in\mathcal V$ and transition index $i$,
\begin{equation}
 \max\left\{\E e^{-\lambda\excessloss_i(\vv,t)},
             \E e^{-\lambda\avgexcess_i(\vv)}\right\}
 \leq e^{-\riskdecay(R(\vv,\vv^\star))},
 \label{eq:averaged-excess-mgf}
\end{equation}
where $\excessloss_i$ and $\avgexcess_i$ are defined in
\eqref{eq:averaged-excess-loss}.
\end{lemma}

\begin{proof}
We first bound the exponential moment at a random flow time using the
conditional noise assumption and the fourth-moment bound. Jensen's inequality
then transfers this bound to the integral over flow time.

Fix $\vv$ and draw a stationary transition and an independent uniform flow
time, giving $(\vh,\vr)$. Write
$\vg=\vv(\vh)-\vv^\star(\vh)$, $G=\norm{\vg}^2$, and
$R=R(\vv,\vv^\star)=\E G$.
Use $\lambda=1/(4\nu^2)$ from \eqref{eq:excess-decay}.
Conditional on $\vh$, the vector $\vg$ is fixed, so
Assumption~\ref{ass:conditional-noise} gives
\begin{align}
 \E\left[e^{-\lambda(G+2\ip{\vg}{\boldsymbol\xi})}\mid\vh\right]
 &\leq e^{-\lambda G+2\lambda^2\nu^2G}
 =e^{-G/(8\nu^2)}.
 \label{eq:single-time-excess-mgf}
\end{align}
If $R>0$, Cauchy--Schwarz and \eqref{eq:l4-l2} yield
\begin{equation*}
 R\leq\frac R2+\E[G\mathbf1\{G\geq R/2\}]
 \leq\frac R2+\sqrt{\E G^2\,\Pp(G\geq R/2)},\qquad
 \Pp(G\geq R/2)\geq\frac{R^2}{4\E G^2}
 \geq\frac1{4C_{4\to2}^4}.
\end{equation*}
Averaging the conditional bound \eqref{eq:single-time-excess-mgf} and using
this probability bound gives
\begin{equation*}
 \E e^{-\lambda(G+2\ip{\vg}{\boldsymbol\xi})}
 \leq\E e^{-G/(8\nu^2)}
 \leq1-\frac{1-e^{-R/(16\nu^2)}}{4C_{4\to2}^4}
 \leq e^{-\riskdecay(R)}.
\end{equation*}
For $R=0$, \eqref{eq:single-time-excess-mgf} directly gives the same bound.
This proves the random-time bound. For the integrated loss, Jensen's
inequality gives, pointwise in the
transition observation,
\begin{equation*}
 e^{-\lambda\avgexcess_i(\vv)}
 =\exp\left(-\lambda\int_0^1\excessloss_i(\vv,t)\dd t\right)
 \leq\int_0^1e^{-\lambda\excessloss_i(\vv,t)}\dd t.
\end{equation*}
Taking expectations and applying Tonelli's theorem proves
\eqref{eq:averaged-excess-mgf}.
\end{proof}

The next lemma uses only the exponential bound for a stationary loss and
the mixing of its observations. We state it for a general loss
$\avgexcess_i(\vv)=\avgexcess(\vv,\vo_i)$ so that it applies to both
integrated and sampled flow times.

\begin{lemma}[Risk from empirical excess loss]\label{lem:excess-risk}
Let $(\vo_i)$ be a stationary Markov process as in Lemma~\ref{lem:coupling},
and let $\mathcal V$ be finite. Fix $\nu>0$, $C_{4\to2}\geq1$, and use $\lambda,\riskdecay$
from \eqref{eq:excess-decay}, with finite risks $R$ as in \eqref{eq:erm-risk}.
Suppose real-valued measurable losses
$\avgexcess_i(\vv)=\avgexcess(\vv,\vo_i)$ satisfy
\begin{equation}
 \E e^{-\lambda\avgexcess_i(\vv)}
 \leq e^{-\riskdecay(R(\vv,\vv^\star))}
 \quad\text{for every }\vv\in\mathcal V.
 \label{eq:stationary-loss-bound}
\end{equation}
Fix an integer $1\leq K\leq m$ and $\varepsilon_{\rm opt}\geq0$. Any possibly
data-dependent $\widetilde{\vv}\in\mathcal V$ with
$m^{-1}\sum_i\avgexcess_i(\widetilde{\vv})\leq\varepsilon_{\rm opt}$
almost surely satisfies, for every $u>0$,
\begin{equation}
 \Pp\left(R(\widetilde{\vv},\vv^\star)\geq u\right)
 \leq m\beta_K+|\mathcal V|
 \exp\left\{\frac mK\left(\lambda\varepsilon_{\rm opt}-\riskdecay(u)\right)\right\}.
 \label{eq:general-estimation-failure}
\end{equation}
\end{lemma}

\begin{proof}
Apply Lemma~\ref{lem:coupling} with spacing $K$ to obtain copies
$(\overline{\vo}_i)$. Let $\overline{\avgexcess}_i(\vv)$ be
$\avgexcess(\vv,\overline{\vo}_i)$.
Each copied loss is a function of its copied observation. Thus, for fixed
$\vv$, independence within $I_a$ and the stationary exponential bound
\eqref{eq:stationary-loss-bound} give
\begin{equation*}
 \E\exp\left(-\lambda\sum_{i\in I_a}\overline{\avgexcess}_i(\vv)\right)
 \leq e^{-|I_a|\riskdecay(R(\vv,\vv^\star))}.
\end{equation*}
To combine the classes, apply H\"older's inequality with all $K$ exponents
equal to $K$ (the identity when $K=1$):
\begin{align*}
 \E\exp\left(-\frac\lambda K\sum_i\overline{\avgexcess}_i(\vv)\right)
 &\leq\prod_{a=0}^{K-1}
 \left[\E\exp\left(-\lambda\sum_{i\in I_a}
           \overline{\avgexcess}_i(\vv)\right)\right]^{1/K}\\
 &\leq\exp\left(-\frac mK\riskdecay(R(\vv,\vv^\star))\right).
\end{align*}
This step requires no independence across residue classes.
Markov's inequality now gives
\begin{equation*}
 \Pp\left(\frac1m\sum_i\overline{\avgexcess}_i(\vv)
                 \leq\varepsilon_{\rm opt}\right)
 \leq\exp\left\{\frac mK
       \left(\lambda\varepsilon_{\rm opt}-\riskdecay(R(\vv,\vv^\star))\right)\right\}.
\end{equation*}
Since $\riskdecay$ is increasing, a union bound over candidates with
$R(\vv,\vv^\star)\geq u$ bounds the probability that any such candidate has
copied excess loss at most $\varepsilon_{\rm opt}$ by the exponential term
in \eqref{eq:general-estimation-failure}.
On the event that all copies agree with the data,
$\overline{\avgexcess}_i(\vv)=\avgexcess_i(\vv)$ for every $i,\vv$.
The assumed empirical loss bound for $\widetilde{\vv}$ then places it among
the candidates just considered whenever its risk is at least $u$.
Adding the coupling failure probability $m\beta_K$ proves the claim.
\end{proof}

\begin{proof}[Proof of Theorem~\ref{thm:velocity-estimation}]
We apply Lemma~\ref{lem:excess-risk} with the integrated excess losses in
\eqref{eq:averaged-excess-loss}. Lemma~\ref{lem:paired-transition-mixing}
supplies the stationary Markov observations, and
Lemma~\ref{lem:averaged-excess} supplies the exponential bound
\eqref{eq:stationary-loss-bound}. It remains to choose the risk threshold.
First suppose $\nu>0$, and set
\begin{equation*}
 K=\tau_{\rm mix}(\delta/(3m)),\qquad
 L=\log(6K|\mathcal V|/\delta),\qquad
 u=128C_{4\to2}^4\nu^2KL/m.
\end{equation*}
The sample-size condition \eqref{eq:estimation-sample-size} implies
$K\leq m$ and $u/(16\nu^2)=8C_{4\to2}^4KL/m\leq1/16$.
By the definition of $\riskdecay$ in \eqref{eq:excess-decay} and
$1-e^{-s}\geq s/2$ for $s\in[0,1]$,
\begin{equation*}
 \riskdecay(u)\geq\frac1{4C_{4\to2}^4}\frac{u}{32\nu^2}
 =\frac{KL}m.
\end{equation*}
Realizability and empirical optimality give
$m^{-1}\sum_i\avgexcess_i(\widehat{\vv})
=\widehat L(\widehat{\vv})-\widehat L(\vv^\star)\leq0$.
Lemma~\ref{lem:excess-risk} with $\varepsilon_{\rm opt}=0$ therefore yields
\begin{equation}
 \Pp\left(R(\widehat{\vv},\vv^\star)\geq u\right)
 \leq m\beta_K+|\mathcal V|e^{-L}
 \leq\frac\delta3+\frac\delta{6K}\leq\frac\delta2.
 \label{eq:estimation-probability}
\end{equation}
This proves the stated $1-\delta$ guarantee. More generally, the same
calculation holds for any $1\leq K\leq m$ with $m\beta_K\leq\delta/3$
and \eqref{eq:estimation-sample-size}.

If $\nu=0$, the conditional noise bound also holds with every positive
parameter $\eta$. The preceding argument then bounds
$\Pp(R(\widehat{\vv},\vv^\star)>128C_{4\to2}^4\eta^2KL/m)$ by $\delta/2$.
As $\eta\downarrow0$, these events increase to
$\{R(\widehat{\vv},\vv^\star)>0\}$.
Continuity of probability proves the zero-risk conclusion.
\end{proof}

\begin{corollary}[Sampled flow times]\label{cor:sampled-flow-times}
Under the assumptions and sample-size condition of
Theorem~\ref{thm:velocity-estimation}, fix an integer $J\geq1$ and draw
$(t_{ij})_{0\leq i<m,\,1\leq j\leq J}$ iid from $\Unif(0,1)$,
independently of the trajectories. Any minimizer $\widehat{\vv}_J$ of
\begin{equation}
 \widehat L_J(\vv)=\frac1{mJ}\sum_{i=0}^{m-1}\sum_{j=1}^J
 \norm{\vv(\vh_i(t_{ij}))-\vr_i}^2
 \label{eq:sampled-time-loss}
\end{equation}
obeys \eqref{eq:estimation-bound}, with $\widehat{\vv}_J$ in place of
$\widehat{\vv}$, with probability at least $1-\delta$ over the trajectories
and flow-time samples.
\end{corollary}

\begin{proof}
First suppose $\nu>0$. Append the independent time vector
$(t_{i1},\ldots,t_{iJ})$ to $\vo_i$.
The resulting process is stationary and Markov. Appending independent
uniform draws cannot increase total variation by
\eqref{eq:tv-kernel-contraction}; contraction under projection onto
$\vo_i$ gives the reverse inequality. Its mixing profile is therefore exactly
$\beta_\ell$. Apply Lemma~\ref{lem:excess-risk} with the loss
$\avgexcess_i(\vv)=J^{-1}\sum_j\excessloss_i(\vv,t_{ij})$.
To verify its exponential hypothesis, convexity gives
\begin{equation*}
 \exp\left(-\frac\lambda J\sum_j\excessloss_i(\vv,t_{ij})\right)
 \leq\frac1J\sum_j e^{-\lambda\excessloss_i(\vv,t_{ij})}.
\end{equation*}
Taking expectations and applying the random-time bound in
Lemma~\ref{lem:averaged-excess} to each summand gives
$\E\exp(-\lambda J^{-1}\sum_j\excessloss_i(\vv,t_{ij}))
\leq e^{-\riskdecay(R(\vv,\vv^\star))}$, with $\lambda,\riskdecay$
from \eqref{eq:excess-decay}. This verifies \eqref{eq:stationary-loss-bound}.
Realizability and empirical optimality give
$m^{-1}\sum_i\avgexcess_i(\widehat{\vv}_J)\leq0$, so Lemma~\ref{lem:excess-risk} applies
with $\varepsilon_{\rm opt}=0$. Use the same $K,L,u$ as in the proof of
Theorem~\ref{thm:velocity-estimation}; its limiting argument also covers
$\nu=0$.
\end{proof}

The corollary controls risk integrated over flow time, for each fixed $J$,
without treating $mJ$ evaluations as independent transitions. A deterministic
time grid instead requires moment and noise assumptions for its time
sampling distribution; transferring grid-averaged risk to integrated risk
also requires control of the time-discretization error.
For approximate minimization of the integrated loss,
$\widehat L(\widetilde{\vv})\leq\min_{\vv\in\mathcal V}\widehat L(\vv)
+\varepsilon_{\rm opt}$ gives
$m^{-1}\sum_i\avgexcess_i(\widetilde{\vv})\leq\varepsilon_{\rm opt}$
by realizability. Lemma~\ref{lem:excess-risk} applies with the integrated
excess loss in \eqref{eq:averaged-excess-loss}.
Thus \eqref{eq:general-estimation-failure} also records the effect of an
empirical optimization error; we do not certify this error for AdamW.

For \eqref{eq:intro-sample-complexity}, set
$C=128C_{4\to2}^4\max\{1,\nu^2\}$ and
$L=\log(6K|\mathcal V|/\delta)$. The condition
$m\geq CKL\max\{1,\epsilon^{-1}\}$ implies both
\eqref{eq:estimation-sample-size} and
$128C_{4\to2}^4\nu^2KL/m\leq\epsilon$.
Theorem~\ref{thm:velocity-estimation} therefore gives risk at most
$\epsilon$ with probability at least $1-\delta$.

\subsection{The sample-size lower bound}\label{app:lower-bound}

\begin{proof}[Proof of Proposition~\ref{prop:mixing-lower-bound}]
We construct two models with separated velocities whose observations agree
when a slowly mixing Gaussian coordinate stays nonnegative. We first verify
the flow and moment assumptions and compute the velocity separation. We then
bound the probability that the data agree and apply a two-point testing
argument, as in \citet[Appendix~C.1, proof of Theorem~1]{nagaraj2020least}.
Finally, we relate the Gaussian coordinate's time scale to the observation
mixing time.

For the prescribed integer $\tau\geq1$, set $a=e^{-1/\tau}$ and
$\sigma_\tau^2=(1-a^2)^{-1}$.
Fix a smooth function $\psi:\R\to[0,1]$ equal to $1$ on $(-\infty,-1]$
and to $0$ on $[0,\infty)$. For $G\sim N(0,1)$, write
$M_2=\E\psi(G)^2\geq\Pp(G\leq-1)>0$ and set
$h=\sqrt{\epsilon/(\pi M_2)}$. Under $\theta\in\{-h,h\}$, let
\begin{equation}
 \vy_i=(C_i,U_i),\qquad
 C_{i+1}=aC_i+\eta_{i+1},\qquad
 U_{i+1}=\theta\psi(C_i/\sigma_\tau)+\zeta_{i+1},
 \label{eq:lower-model}
\end{equation}
where $(\eta_i)$ and $(\zeta_i)$ are independent iid standard Gaussian
sequences. Use a stationary extension of $(C_i)$ to index $-1$ and set
$U_0=\theta\psi(C_{-1}/\sigma_\tau)+\zeta_0$.
Then $\vy$ is stationary and $C_i/\sigma_\tau\sim N(0,1)$.
Independently, take $\vx_i\overset{\mathrm{iid}}{\sim}N(\vzero,\Id_2)$.
For $\vy=(c,u)$, put
$\vg_\theta(\vy)=(ac,\theta\psi(c/\sigma_\tau))$.
The conditional transitions are
$P_X(\vx,\cdot)=N(\vzero,\Id_2)$ and
$P_{Y,\theta}(\vy,\cdot)=N(\vg_\theta(\vy),\Id_2)$.

For fixed $(\vx,\vy)$, apply Lemma~\ref{lem:gaussian-flow} with means
$\vzero,\vg_\theta(\vy)$ in dimension two. With $s_t=(1-t)^2+t^2$, it gives
\begin{equation}
 \vv_{\theta,t}^\star(\vz,\vx,\vy)
 =\frac{2t-1}{s_t}\vz+\frac{1-t}{s_t}\vg_\theta(\vy),\qquad
 \mathcal V=\{\vv_{-h}^\star,\vv_h^\star\}.
 \label{eq:lower-velocity}
\end{equation}
Its flow is
$\Phi^{\vv_\theta^\star}_{\vx,\vy,t}(\vz_0)
=t\vg_\theta(\vy)+\sqrt{s_t}\,\vz_0$.
Lemma~\ref{lem:gaussian-flow} verifies Assumption~\ref{ass:flow-regularity}
for both candidates, with $L_t=|2t-1|/s_t$ and
$\int_0^1L_t\dd t=\log2$. It also gives
$\boldsymbol\xi\mid\vh\sim N(\vzero,\Id_2/s_t)$, verifying
Assumption~\ref{ass:conditional-noise} with $\nu^2=2$.

Write $R_\theta$ for risk under model $\theta$. Only the second velocity
coordinate differs, by $2h(1-t)\psi(c/\sigma_\tau)/s_t$ in absolute value.
Using the integral \eqref{eq:gaussian-time-integrals},
\begin{equation}
 R_\theta(\vv_{-\theta}^\star,\vv_\theta^\star)
 =4h^2M_2\int_0^1\frac{(1-t)^2}{s_t^2}\dd t
 =\pi h^2M_2=\epsilon.
 \label{eq:lower-separation}
\end{equation}
The same class contains the truth under either model. Since the candidate
difference has norm at most $4h$,
\begin{equation*}
 \E_\theta\norm{\vv_{-\theta}^\star(\vh)-\vv_\theta^\star(\vh)}^4
 \leq16h^2 R_\theta(\vv_{-\theta}^\star,\vv_\theta^\star)
 =\frac{16}{\pi M_2}
 R_\theta(\vv_{-\theta}^\star,\vv_\theta^\star)^2.
\end{equation*}
This verifies Assumption~\ref{ass:velocity-class} with the universal constant
$C_{4\to2}^4=16/(\pi M_2)$.

The candidates thus satisfy the flow and moment assumptions and have risk
separation $\epsilon$. To show that a short trajectory cannot reliably
distinguish them, we exhibit an event on which their observations coincide.
Couple the two models using the same source trajectory, Gaussian noises,
$C$-trajectory, and auxiliary randomness of the estimator. On
$E=\{C_{-1},C_0,\ldots,C_{m-1}\geq0\}$, the observed
$(U_0,\ldots,U_m)$ equal $(\zeta_0,\ldots,\zeta_m)$ under both models.
Thus the data agree on $E$, including the initial observation $U_0$.
Let $(Z_s)_{s\geq0}$ be a stationary Ornstein--Uhlenbeck process with
covariance $\E[Z_sZ_t]=e^{-|s-t|}$. The vectors
$(C_{-1},\ldots,C_{m-1})/\sigma_\tau$ and
$(Z_0,Z_{1/\tau},\ldots,Z_{m/\tau})$ have the same centered Gaussian law.
Continuous nonnegativity implies nonnegativity at these sampled times.
The survival formula in \citet[Equation~(2.5)]{aurzada2011survival}, evaluated
at time $2m/\tau$ to account for their covariance $e^{-|s-t|/2}$, gives
\begin{equation}
 \Pp(E)\geq\Pp(Z_s\geq0\text{ for all }s\in[0,m/\tau])
 =\frac1\pi\arcsin(e^{-m/\tau})
 \geq\frac1\pi e^{-m/\tau}.
 \label{eq:lower-no-visit}
\end{equation}
Since the estimator takes values in $\mathcal V=\{\vv_{-h}^\star,\vv_h^\star\}$,
let $p$ be its conditional probability of returning $\vv_h^\star$ given $E$.
The conditional data law is the same under both models, so the conditional
error probabilities are $p$ under $\theta=-h$ and $1-p$ under $\theta=h$.
Their maximum is at least $1/2$.
Together with \eqref{eq:lower-separation}, the two-point argument gives
\begin{equation}
 \max_{\theta\in\{-h,h\}}
 \Pp_\theta\left(R_\theta(\widetilde{\vv},\vv_\theta^\star)\geq\epsilon\right)
 \geq\frac12\Pp(E)\geq\frac1{2\pi}e^{-m/\tau}.
 \label{eq:mixing-lower-bound}
\end{equation}

It remains to verify that $\tau$ is proportional to the observation mixing
time. Write
$\tau_{{\rm mix},\theta}(\alpha)$ for the observation mixing time
under model $\theta$, as defined in \eqref{eq:observation-mixing}.
Let $\pi_C=N(0,\sigma_\tau^2)$ and let $\beta_k^C$ be the
stationary-average coefficient of $(C_i)$. The joint-law identity
\eqref{eq:joint-mixing-tv}, applied to this chain, gives
$\beta_k^C=\norm{\Law(C_0,C_k)-\pi_C\otimes\pi_C}_{\TV}$.
For $r=a^k$, the relative entropy of the joint standardized Gaussian pair
$(C_0/\sigma_\tau,C_k/\sigma_\tau)$ with respect to two independent standard
Gaussians is $-\frac12\log(1-r^2)$, as computed in
\eqref{eq:gaussian-mixing-entropy}. Pinsker's inequality
\citep[Lemma~2]{canonne2022kl} gives the upper bound
below. For the lower bound, the bivariate Gaussian sign formula
\citep[proof of Lemma~2.11]{aurzada2011survival} shows that the probability
that $C_0,C_k$ have the same sign exceeds its value under independence by
$\arcsin(a^k)/\pi$. Thus, for $k\geq1$,
\begin{equation*}
 \frac1\pi\arcsin(a^k)\leq\beta_k^C
 \leq\frac12\sqrt{-\log(1-a^{2k})}.
\end{equation*}
For $\ell\geq3$, conditional on $\vo_0$, the law of $C_{\ell-1}$ depends
only on $C_1$ and is the $(\ell-2)$-step Gaussian transition law.
Given $C_{\ell-1}$, draw $C_\ell,C_{\ell+1}$ and the independent noises
defining $U_\ell,U_{\ell+1}$, as well as the independent source pair.
This constructs $\vo_\ell$ by a common Markov kernel.
The contraction bound \eqref{eq:tv-kernel-contraction} gives
the upper bound below; projecting $(\vo_0,\vo_\ell)$ onto
$(C_1,C_\ell)$ gives the lower bound:
\begin{equation}
 \beta_{\ell-1}^C\leq\beta_\ell\leq\beta_{\ell-2}^C.
 \label{eq:lower-bound-mixing}
\end{equation}
The lower bound also holds for $\ell=2$, and $\beta_1=1$ since adjacent
observations share the continuous coordinate $C_1$.
Since $\arcsin(e^{-1/4})/\pi>1/4$ and
$\sqrt{-\log(1-e^{-2})}/2<1/4$, these bounds imply, under both models,
\begin{equation*}
 \frac\tau4\leq\tau_{{\rm mix},\theta}(1/4)\leq\tau+2\leq3\tau.
\end{equation*}
The source is iid, and the target is a projection of the mixing observation
process, so both satisfy Model~\ref{model:stationary-training}. With the
mixing-time comparison established, an estimator that succeeds under both
models with probability at least $1-\delta$ must satisfy, by
\eqref{eq:mixing-lower-bound},
\begin{equation*}
 m\geq\tau\log\frac1{2\pi\delta}
 \geq\frac13\tau_{{\rm mix},\theta}(1/4)\log\frac1{2\pi\delta}
 \quad\text{for either }\theta,
\end{equation*}
which proves \eqref{eq:mixing-lower-sample}.
\end{proof}

\subsection{Comparison of mixing tolerances}\label{app:mixing-comparison}

\begin{remark}[Comparing the upper and lower bounds]
To compare the bounds, we relate the estimation spacing
$K=\tau_{\rm mix}(\delta/(3m))$ to $\tau_{\rm mix}$. If
$\beta_\ell\leq c_0\exp(-c_1\ell/\tau_{\rm mix})$ for all $\ell\geq1$
and universal $c_0\geq1,c_1>0$, then
\begin{equation*}
 \tau_{\rm mix}(\delta/(3m))
 \leq\left\lceil\frac{\tau_{\rm mix}}{c_1}
 \log\frac{3c_0m}{\delta}\right\rceil.
\end{equation*}
If we strengthen the mixing assumption in Model~\ref{model:stationary-training}
to require finite worst-case mixing time $\tau_{\rm wc}(1/4)$ for the
observation process $(\vo_i)$, with $\tau_{\rm wc}$ defined in
Section~\ref{sec:preliminaries},
then
\begin{equation*}
 \tau_{\rm mix}(\delta/(3m))\leq\tau_{\rm wc}(\delta/(3m))
 \leq \underbrace{\tau_{\rm wc}(1/4)
 \left\lceil\log_2\frac{3m}{\delta}\right\rceil}_{K}
\end{equation*}
by total-variation contraction
\citep[Lemmas~4.10--4.11 and (4.34)]{levin2017markov}.
This choice satisfies $m\beta_K\leq\delta/3$. The proof of
Theorem~\ref{thm:velocity-estimation} therefore applies with this $K$
in both its risk bound and sample-size condition.
The Gaussian models in Proposition~\ref{prop:mixing-lower-bound} satisfy
the exponential stationary-average profile above, by
\eqref{eq:lower-bound-mixing} and its Gaussian bounds.
Their worst-case mixing time is infinite: conditional on $C_1=c$,
the mean of $C_\ell$ is $a^{\ell-1}c$, which is unbounded as $|c|$ grows,
while its conditional variance is fixed. Thus $d_\ell^{\rm wc}=1$ at every
finite lag.
\end{remark}

\subsection{Accuracy of generated states and trajectories}\label{sec:stability}\label{app:stability-proofs}\label{sec:finite-rollout}

The estimation theorem controls
velocity error at stationary training pairs, but the learned recursion
visits pairs drawn from its own rollout distribution. We connect these
errors through the concentrability condition
\eqref{eq:rollout-concentrability}, then use contraction of the target
transitions to control their accumulation across steps.
Throughout this subsection, laws and expectations involving a fitted
velocity are taken with the training data held fixed.
For a candidate velocity $\vv\in\mathcal V$, let $\widehat{\vy}_i^{\vv}$ be its rollout
in \eqref{eq:population-sampler} using $\Phi^{\vv}$ and an initialization
$\widehat{\vy}_0\sim\widehat\mu_0$ independent of the fresh source and training data.
Write $\Gamma_i^{\vv}=\Law(\vx'_i,\widehat{\vy}_i^{\vv})$ for the
distribution of the source and generated state pair. The velocity error
at a fixed pair is
\begin{equation}
 \ell_{\vv}(\vx,\vy)=\int_0^1\E_{q_{t\mid\vx,\vy}}
 \norm{\vv_t(\vz,\vx,\vy)-\vv_t^\star(\vz,\vx,\vy)}^2\dd t,
 \qquad R(\vv,\vv^\star)=\E_{\pi_X\otimes\pi_Y}\ell_{\vv}.
 \label{eq:conditional-velocity-risk}
\end{equation}
The next theorem assumes that the error under $\Gamma_i^{\vv}$ is
controlled by its training average and combines this with target contraction.
\begin{theorem}[Finite-sample rollout]\label{thm:finite-rollout}
Under the assumptions of Theorem~\ref{thm:velocity-estimation} and
Assumption~\ref{ass:flow-regularity}, fix $\epsilon>0$, $\delta\in(0,1)$,
and an integer $n\geq1$. Suppose the candidate rollouts are well-defined,
every candidate is $\widehat L_t$-Lipschitz in $\vz$, uniformly over
current-state pairs, for a common nonnegative function $\widehat L_t$
with $\Lambda=\int_0^1\widehat L_t\dd t<\infty$. Suppose, for some finite
$C_{\rm roll}\geq1$,
\begin{equation}
 \E_{\Gamma_i^{\vv}}\ell_{\vv}\leq C_{\rm roll}R(\vv,\vv^\star)
 \qquad\text{for all }\vv\in\mathcal V,\quad 0\leq i<n.
 \label{eq:rollout-concentrability}
\end{equation}
Suppose the target transitions and $\pi_Y$ have finite second moments,
$a_0=W_2(\widehat\mu_0,\pi_Y)<\infty$, and
$W_2(P_Y(\vy,\cdot),P_Y(\vy',\cdot))\leq\rho\norm{\vy-\vy'}$
for some $0\leq\rho<1$. If $m$ satisfies \eqref{eq:intro-sample-complexity},
then, with probability at least $1-\delta$ over the training trajectories,
the rollout using the exact ODE flow of $\widehat{\vv}$ satisfies
\begin{equation}
 W_2(\Law(\widehat{\vy}_n),\pi_Y)
 \leq\rho^na_0+\frac{e^\Lambda\sqrt{C_{\rm roll}\epsilon}}{1-\rho}.
 \label{eq:finite-rollout}
\end{equation}
\end{theorem}
The bound separates the decaying initialization error from the accumulated
estimation error, whose amplification by the target dynamics is at most
$1/(1-\rho)$.
The concentrability condition \eqref{eq:rollout-concentrability} transfers
training accuracy to generated state pairs, as in offline reinforcement
learning \citep[Section~4.3, (16b)]{zanette2022bellman}.

Proposition~\ref{prop:one-transition} converts conditional velocity error
to transition error, and Proposition~\ref{prop:rollout} propagates this error
along the trajectory. We condition on the fitted velocity when applying
these deterministic bounds, then use Theorem~\ref{thm:velocity-estimation}
to prove the finite-sample guarantee.

\begin{assumption}[Learned velocity Lipschitz bound]\label{ass:learned-lipschitz}
There is a nonnegative integrable function $\widehat L_t$ such that, for
every current-state pair used in the result, all $\vz,\vz'\in\R^d$, and
almost every $t\in[0,1]$,
\begin{equation}
 \norm{\widehat{\vv}_t(\vz,\vx,\vy)-
 \widehat{\vv}_t(\vz',\vx,\vy)}
 \leq\widehat L_t\norm{\vz-\vz'},
 \qquad
 \Lambda=\int_0^1\widehat L_t\dd t<\infty.
 \label{eq:learned-lipschitz}
\end{equation}
\end{assumption}

A uniform Lipschitz bound on the learned velocity is also used by
\citet[Assumption~1(b)]{bansal2026straightness}; here the bound may vary with
flow time. Every candidate in Appendix~\ref{app:gaussian} satisfies it with
$\Lambda=\log2$.

\begin{proposition}[One-transition error]\label{prop:one-transition}
Under Assumptions~\ref{ass:flow-regularity} and~\ref{ass:learned-lipschitz}
at $(\vx,\vy)$, the learned transition map satisfies
\begin{equation}
 W_2^2\left(
 (\Phi^{\widehat{\vv}}_{\vx,\vy})_\#P_X(\vx,\cdot),
 P_Y(\vy,\cdot)\right)
 \leq e^{2\Lambda}\ell_{\widehat{\vv}}(\vx,\vy),
 \label{eq:one-transition}
\end{equation}
where $\ell_{\vv}$ is the conditional risk in
\eqref{eq:conditional-velocity-risk}.
\end{proposition}

\begin{proof}[Proof of Proposition~\ref{prop:one-transition}]
We start the two flows from the same source draw and bound their separation
using Gronwall's inequality. The population transport identity then converts
this coupling cost to the claimed Wasserstein bound.

Fix current-state values $(\vx,\vy)$ and start both flows from the same
$\vx^+\sim P_X(\vx,\cdot)$. All laws and expectations below are over
this source draw, with $(\vx,\vy)$ held fixed. Write
$\vz_t^\star=\Phi^{\vv^\star}_{\vx,\vy,t}(\vx^+)$,
$\widehat{\vz}_t=\Phi^{\widehat{\vv}}_{\vx,\vy,t}(\vx^+)$, and
$b_t=\norm{\widehat{\vv}_t(\vz_t^\star,\vx,\vy)-\vv_t^\star(\vz_t^\star,\vx,\vy)}$.
Subtract the two ODEs and apply Assumption~\ref{ass:learned-lipschitz} to obtain
\begin{equation*}
 \norm{\widehat{\vz}_t-\vz_t^\star}
 \leq\int_0^t\left(
 \widehat L_s\norm{\widehat{\vz}_s-\vz_s^\star}+b_s\right)\dd s.
\end{equation*}
Gronwall's inequality and Cauchy--Schwarz give
\begin{equation*}
 \E\norm{\widehat{\vz}_1-\vz_1^\star}^2
 \leq\E\left(\int_0^1 e^{\int_t^1\widehat L_s\dd s}b_t\dd t\right)^2
 \leq e^{2\Lambda}\int_0^1\E b_t^2\dd t.
\end{equation*}
To express the right side in terms of conditional velocity error,
Proposition~\ref{prop:conditional-transport} gives
$\Law(\vz_t^\star)=q_{t\mid\vx,\vy}$, so the right side is
$e^{2\Lambda}\ell_{\widehat{\vv}}(\vx,\vy)$. At $t=1$, the coupled
endpoints have the two laws in \eqref{eq:one-transition}, proving the bound.
\end{proof}

For a fixed measurable transition map $\Phi_{\vx,\vy}$ and its rollout
\eqref{eq:population-sampler}, let
$\widetilde P_{\vx,\vy}=(\Phi_{\vx,\vy})_\#P_X(\vx,\cdot)$ and define
\begin{equation}
 e_i^2=\E W_2^2\left(\widetilde P_{\vx'_i,\widehat{\vy}_i},
 P_Y(\widehat{\vy}_i,\cdot)\right).
 \label{eq:on-policy-error}
\end{equation}
To propagate this error, we assume that the target kernel is Lipschitz in $W_2$.

\begin{assumption}[Target kernel Lipschitz bound]\label{ass:target-lipschitz}
The target transition distributions have finite second moments, and there is
a finite constant $\rho\geq0$ such that, for all $\vy,\vy'\in\R^d$,
\begin{equation}
 W_2(P_Y(\vy,\cdot),P_Y(\vy',\cdot))
 \leq\rho\norm{\vy-\vy'}.
\label{eq:contractive}
\end{equation}
\end{assumption}

For example, consider transitions $\vg(\vy)+\boldsymbol\eta$, where
$\vg:\R^d\to\R^d$ is $\rho$-Lipschitz and the noise law is independent of
$\vy$ with $\E\norms{\boldsymbol\eta}^2<\infty$. Use the same noise at both
states. This gives
$W_2(P_Y(\vy,\cdot),P_Y(\vy',\cdot))\leq
\norms{\vg(\vy)-\vg(\vy')}\leq\rho\norms{\vy-\vy'}$.

\begin{proposition}[Rollout recursion]\label{prop:rollout}
Under Model~\ref{model:markov-transport} and
Assumption~\ref{ass:target-lipschitz}, suppose $P_Y$ has a stationary
distribution $\pi_Y$ with finite second moment. Fix measurable transition
maps $\Phi_{\vx,\vy}$ and generate by \eqref{eq:population-sampler}, with
$\widehat{\vy}_0\sim\widehat\mu_0$ independently of the fresh source.
Suppose $a_0=W_2(\widehat\mu_0,\pi_Y)<\infty$ and the errors $e_i$ in
\eqref{eq:on-policy-error} are finite. With $a_{i+1}=\rho a_i+e_i$,
\begin{equation*}
 W_2(\Law(\widehat{\vy}_n),\pi_Y)\leq a_n
 =\rho^na_0+\sum_{i=0}^{n-1}\rho^{n-1-i}e_i.
\end{equation*}
For a stationary target trajectory $\vy_{0:n}$, the joint laws satisfy
\begin{equation}
 W_{2,d_n}^2\left(\Law(\widehat{\vy}_{0:n}),\Law(\vy_{0:n})\right)
 \leq\sum_{i=0}^n a_i^2,
 \qquad d_n^2(\vy,\vz)=\sum_{i=0}^n\norm{\vy_i-\vz_i}^2,
 \label{eq:trajectory-bound}
\end{equation}
where $W_{2,d_n}$ uses the trajectory metric $d_n$.
\end{proposition}

\begin{proof}[Proof of Proposition~\ref{prop:rollout}]
We couple the generated and stationary target trajectories while preserving
the source transitions. We then bound the coupling cost recursively, iterate
for the marginal guarantee, and sum the costs for the trajectory guarantee.

Couple
$(\widehat{\vy}_0,\vy_0)$ optimally, independently of the fresh source.
At each step, let $\mathcal H_i$ be the joint history of the source,
generated, and comparison trajectories through time $i$, including the
auxiliary randomness already used. We construct this history from
$\vx'_{0:i}$ and randomness independent of the full source trajectory.
The Markov and tower-property argument of
Lemma~\ref{lem:current-state-sufficiency} therefore gives
\begin{equation*}
 \Law(\vx'_{i+1}\mid\mathcal H_i)=P_X(\vx'_i,\cdot),\qquad
 \Law(\widehat{\vy}_{i+1}\mid\mathcal H_i)
 =\widetilde P_{\vx'_i,\widehat{\vy}_i}.
\end{equation*}
Keep this source transition and its generated output. Conditional on
$\mathcal H_i$, couple that output
optimally to $\vu_{i+1}\sim P_Y(\widehat{\vy}_i,\cdot)$, and couple
$\vu_{i+1}$ optimally to $\vy_{i+1}\sim P_Y(\vy_i,\cdot)$.
Measurable optimal couplings, disintegration, and gluing on $\R^d$
justify these conditional draws
\citep[Chapter~1 and Corollary~5.22]{villani2009optimal}.
They use only the history, the generated output, and fresh auxiliary
randomness, so future source transitions are unchanged. The comparison
process has conditional transition law $P_Y(\vy_i,\cdot)$ and hence the
stationary target trajectory law.

To bound the cost of this coupling, let
$\Delta_i=(\E\norm{\widehat{\vy}_i-\vy_i}^2)^{1/2}$.
By the definition of $e_i$ in \eqref{eq:on-policy-error} and
Assumption~\ref{ass:target-lipschitz}, the two costs satisfy
\begin{equation*}
 \E\norm{\widehat{\vy}_{i+1}-\vu_{i+1}}^2=e_i^2,
 \qquad
 \E\left[\norm{\vu_{i+1}-\vy_{i+1}}^2\mid\mathcal H_i\right]
 \leq\rho^2\norm{\widehat{\vy}_i-\vy_i}^2.
\end{equation*}
Minkowski's inequality now gives
\begin{align*}
 \Delta_{i+1}
 &\leq\left(\E\norm{\widehat{\vy}_{i+1}-\vu_{i+1}}^2\right)^{1/2}
 +\left(\E\norm{\vu_{i+1}-\vy_{i+1}}^2\right)^{1/2}\\
 &\leq e_i+\rho\Delta_i.
\end{align*}
Since $\Delta_0=a_0$, induction gives $\Delta_i\leq a_i$, and solving the
recursion for $a_i$ yields the marginal bound. For the joint law, the same coupling has
expected trajectory cost
$\E\sum_{i=0}^n\norm{\widehat{\vy}_i-\vy_i}^2
=\sum_{i=0}^n\Delta_i^2\leq\sum_{i=0}^n a_i^2$,
which proves \eqref{eq:trajectory-bound}.
\end{proof}

The concentrability condition \eqref{eq:rollout-concentrability} is uniform
over $\mathcal V$, so it applies to the fitted candidate. The constant
$C_{\rm roll}$ may depend on the horizon, source
initialization, and target initialization; a bound uniform in the horizon
requires the same constant for all $i$. Appendix~\ref{app:related-rollout}
gives the precise comparison with off-policy evaluation. Verifying this
condition for broader model classes remains an open question.

For example, if $\Gamma_i^{\vv}$ has density at most $M$ relative to
$\pi_X\otimes\pi_Y$, uniformly over candidates and times, then
\eqref{eq:rollout-concentrability} holds with $C_{\rm roll}=M$.

\begin{proof}[Proof of Theorem~\ref{thm:finite-rollout}]
We use the estimation theorem to bound population risk, then concentrability
to control transition error at generated states. The rollout recursion and
a geometric sum give the result.

Under \eqref{eq:intro-sample-complexity}, Theorem~\ref{thm:velocity-estimation}
gives
$R(\widehat{\vv},\vv^\star)\leq\epsilon$ with probability at least
$1-\delta$, as verified in Appendix~\ref{app:estimation-proofs}.
Fix any training sample in this event and use the maps
$\Phi_{\vx,\vy}=\Phi^{\widehat{\vv}}_{\vx,\vy}$ in
\eqref{eq:on-policy-error}. The common candidate Lipschitz bound gives
Assumption~\ref{ass:learned-lipschitz} for the selected velocity.
Proposition~\ref{prop:one-transition} and
the concentrability condition \eqref{eq:rollout-concentrability} give,
for every $i<n$,
\begin{equation}
 e_i^2
 \leq e^{2\Lambda}\E_{\Gamma_i^{\widehat{\vv}}}
 \ell_{\widehat{\vv}}(\vx,\vy)
 \leq e^{2\Lambda}C_{\rm roll}R(\widehat{\vv},\vv^\star).
 \label{eq:rollout-local-error}
\end{equation}
These errors are finite at every $i<n$, and the target contraction
assumption gives Assumption~\ref{ass:target-lipschitz}. Thus
Proposition~\ref{prop:rollout}, with the same source and initialization, gives
\begin{equation}
 W_2(\Law(\widehat{\vy}_n),\pi_Y)
 \leq\rho^na_0+
 e^\Lambda\sqrt{C_{\rm roll}R(\widehat{\vv},\vv^\star)}
 \sum_{j=0}^{n-1}\rho^j.
 \label{eq:rollout-from-risk}
\end{equation}
For $0\leq\rho<1$, the geometric sum is
$(1-\rho^n)/(1-\rho)\leq(1-\rho)^{-1}$. Substituting
$R(\widehat{\vv},\vv^\star)\leq\epsilon$ proves
\eqref{eq:finite-rollout}.
No additional failure probability is incurred.
\end{proof}

\section{Gaussian transitions, estimation, and rollout}\label{app:gaussian}

\label{sec:gaussian-example}
This example verifies the estimation assumptions explicitly and bounds
rollout error by coupling the two recursions with the same Gaussian noise.
We first record the conditional calculation also used in the lower bound.

\begin{lemma}[Flow between Gaussian transitions]\label{lem:gaussian-flow}
Fix $\va,\vb\in\R^d$, draw independent
$\vx^+\sim N(\va,\Id_d)$ and $\vy^+\sim N(\vb,\Id_d)$, and set
$\vz_t=(1-t)\vx^++t\vy^+$, $\vr=\vy^+-\vx^+$, and
$s_t=(1-t)^2+t^2$. The conditional mean velocity and its ODE flow are
\begin{equation}
 \vv_t^\star(\vz)=\frac{2t-1}{s_t}\vz-\frac{t}{s_t}\va
                   +\frac{1-t}{s_t}\vb,
 \qquad
 \Phi_t(\vz_0)=(1-t)\va+t\vb+\sqrt{s_t}(\vz_0-\va).
 \label{eq:gaussian-conditional-flow}
\end{equation}
For each $t\in[0,1]$, the residual $\vv_t^\star(\vz_t)-\vr$ conditional
on $\vz_t$ has law $N(\vzero,\Id_d/s_t)$. The velocity satisfies
Assumption~\ref{ass:flow-regularity}. For each $t\in[0,1]$, the map
$\vz\mapsto\vv_t^\star(\vz)$ is Lipschitz with constant
$L_t=|2t-1|/s_t$, and $\int_0^1 L_t\dd t=\log2$.
\end{lemma}

\begin{proof}
Write $\vm_t=(1-t)\va+t\vb$. The pair $(\vr,\vz_t)$ is jointly Gaussian,
with means $\vb-\va,\vm_t$, variances $2\Id_d,s_t\Id_d$, and cross
covariance $(2t-1)\Id_d$. Gaussian conditioning gives
\begin{equation*}
 \E[\vr\mid\vz_t=\vz]
 =\vb-\va+\frac{2t-1}{s_t}(\vz-\vm_t),
 \qquad
 \operatorname{Cov}(\vr\mid\vz_t)
 =\left(2-\frac{(2t-1)^2}{s_t}\right)\Id_d=\frac{\Id_d}{s_t}.
\end{equation*}
Expanding the mean proves the velocity formula. Since
$\dot s_t=2(2t-1)$, differentiating the proposed flow gives
\begin{equation*}
 \frac{\dd}{\dd t}\Phi_t(\vz_0)
 =\vb-\va+\frac{2t-1}{s_t}(\Phi_t(\vz_0)-\vm_t)
 =\vv_t^\star(\Phi_t(\vz_0)),\qquad \Phi_0(\vz_0)=\vz_0.
\end{equation*}
For all $\vz,\vz'\in\R^d$,
\begin{equation*}
 \norm{\vv_t^\star(\vz)-\vv_t^\star(\vz')}
 =\frac{|2t-1|}{s_t}\norm{\vz-\vz'}.
\end{equation*}
The affine coefficients are bounded because $s_t\geq1/2$, and
\begin{equation*}
 \int_0^1\frac{|2t-1|}{s_t}\dd t
 =2\int_{1/2}^1\frac{2t-1}{s_t}\dd t
 =\log s_1-\log s_{1/2}=\log2.
\end{equation*}
These bounds verify Assumption~\ref{ass:flow-regularity}.
\end{proof}

For $s_t=(1-t)^2+t^2$, symmetry under $t\mapsto1-t$ gives the first
two equalities below, and the substitution $q=2t-1$ gives the last:
\begin{equation}
 \int_0^1\frac{t^2}{s_t^2}\dd t
 =\int_0^1\frac{(1-t)^2}{s_t^2}\dd t
 =\frac12\int_0^1\frac{\dd t}{s_t}=\frac\pi4.
 \label{eq:gaussian-time-integrals}
\end{equation}

Consider independent stationary processes
$X_{i+1}=aX_i+\epsilon_{i+1}$ and
$Y_{i+1}=bY_i+\zeta_{i+1}$ with independent standard Gaussian noise.
The unknown $(a,b)$ belongs to a known finite set
$\Theta_0\subset[-r,r]^2$, where $0<r<1$. With
$s_t=(1-t)^2+t^2$, define
\begin{equation}
 \vv_t^{u,w}(z,x,y)=\frac{2t-1}{s_t}z-
 \frac{t}{s_t}ux+\frac{1-t}{s_t}wy,
 \qquad (u,w)\in\Theta_0.
 \label{eq:gaussian-class}
\end{equation}
Its exact time-one map is $\Phi^{u,w}_{x,y}(z)=z-ux+wy$.

\begin{proposition}[Gaussian estimation and rollout]\label{prop:gaussian-complete}
In the Gaussian model above, fix $\delta\in(0,1)$, set
$K=\tau_{\rm mix}(\delta/(3m))$, and suppose
$m\geq1024K\log(6K|\Theta_0|/\delta)$. With probability at least
$1-\delta$, the minimizer of \eqref{eq:integrated-empirical-loss} over
\eqref{eq:gaussian-class} satisfies
\begin{equation}
 R(\widehat{\vv},\vv^\star)\leq2048
 \frac{K\log(6K|\Theta_0|/\delta)}m.
 \label{eq:gaussian-risk}
\end{equation}
On this event, conditional on the training data, the rollout using its exact
ODE flow, a fresh stationary source, and an independent initialization
$\widehat Y_0\sim\widehat\mu_0$ with
$a_0=W_2(\widehat\mu_0,\pi_Y)<\infty$ satisfies, for every integer $n\geq0$,
\begin{equation}
 W_2(\Law(\widehat Y_n),\pi_Y)
 \leq r^na_0+
 \sqrt{\frac8\pi R(\widehat{\vv},\vv^\star)}\sum_{j=0}^{n-1}r^j.
 \label{eq:gaussian-rollout}
\end{equation}
\end{proposition}

More generally, for $P_X(x,\cdot)=N(ax,\sigma_X^2)$ and
$P_Y(y,\cdot)=N(by,\sigma_Y^2)$ with $\sigma_X,\sigma_Y>0$, the population map is
\begin{equation}
 \Phi_{x,y}(z)=by+\frac{\sigma_Y}{\sigma_X}(z-ax).
 \label{eq:gaussian-oracle}
\end{equation}
The scalar population flow is increasing and transports the source law to
the target law, so it agrees with this quantile map.

\begin{proof}[Proof of Proposition~\ref{prop:gaussian-complete}]
We first compute the population flow, then verify the moment and mixing
assumptions to apply Theorem~\ref{thm:velocity-estimation}. For the rollout
bound, we couple the recursions using the same Gaussian noise and use the
risk identity to control their separation.

Apply Lemma~\ref{lem:gaussian-flow} in dimension one with means $ax,by$.
The population velocity is \eqref{eq:gaussian-class} with $(u,w)=(a,b)$,
and the conditional noise variance is $1/s_t\leq2$. Applying the flow
formula with means $ux,wy$ gives, for every candidate,
\begin{equation}
 z_t=(1-t)ux+twy+\sqrt{s_t}(z_0-ux),\qquad z_1=z_0-ux+wy.
 \label{eq:gaussian-solution}
\end{equation}
Lemma~\ref{lem:gaussian-flow} also verifies flow regularity and
Assumption~\ref{ass:learned-lipschitz} for every candidate, with $\Lambda=\log2$.

To check the fourth-moment assumption, fix $(u,w)\in\Theta_0$ and write
$R=R(\vv^{u,w},\vv^{a,b})$, $A=(u-a)^2/(1-a^2)$, and
$B=(w-b)^2/(1-b^2)$. At fixed $t$, the prediction error under
$\pi_X\otimes\pi_Y$ is centered Gaussian with variance
$f_t=[t^2A+(1-t)^2B]/s_t^2$. The integrals in
\eqref{eq:gaussian-time-integrals} give
\begin{equation}
 R(\vv^{u,w},\vv^{a,b})=\frac\pi4(A+B).
 \label{eq:gaussian-risk-identity}
\end{equation}
To bound the fourth moment by $R^2$, use $s_t\geq1/2$ and
$t^2,(1-t)^2\leq s_t$ to obtain $f_t\leq2(A+B)$. Thus
\begin{equation*}
 \E|\vv^{u,w}(\vh)-\vv^{a,b}(\vh)|^4
 =3\int_0^1 f_t^2\dd t
 \leq6(A+B)R=\frac{24}\pi R^2\leq8R^2.
\end{equation*}
Together with the conditional noise calculation above, this verifies
Assumptions~\ref{ass:conditional-noise} and~\ref{ass:velocity-class}
with $\nu^2=2$ and $C_{4\to2}^4=8$; realizability follows from
$(a,b)\in\Theta_0$.

It remains to check stationary mixing before applying the estimation theorem.
Lemma~\ref{lem:paired-transition-mixing} gives
$\beta_\ell\leq\beta_{\ell-1}^X+\beta_{\ell-1}^Y$ for $\ell\geq2$.
To bound $\beta_\ell^X$, use the identity
\begin{equation}
 \beta_\ell^X
 =\norm{\Law(X_0,X_\ell)-\pi_X\otimes\pi_X}_{\TV}.
 \label{eq:joint-mixing-tv}
\end{equation}
Indeed, the two joint laws have the same first marginal $\pi_X$;
integrating the total variation of their conditional second-coordinate
laws gives the definition of $\beta_\ell^X$.
Write $r_\ell=a^\ell$. Standardizing the two coordinates gives a centered
Gaussian pair with covariance
$\mSigma_\ell=\left(\begin{smallmatrix}1&r_\ell\\r_\ell&1\end{smallmatrix}\right)$.
For laws $\mu,\nu$, write $\KL(\mu\Vert\nu)$ for relative entropy.
The Gaussian density formula gives
\begin{equation}
 \KL\left(N(\vzero,\mSigma_\ell)\Vert N(\vzero,\Id_2)\right)
 =\frac12\left(\operatorname{tr}\mSigma_\ell-2-\log\det\mSigma_\ell\right)
 =-\frac12\log(1-r_\ell^2).
 \label{eq:gaussian-mixing-entropy}
\end{equation}
Pinsker's inequality, $\norm{\mu-\nu}_{\TV}^2\leq
\KL(\mu\Vert\nu)/2$ \citep[Lemma~2]{canonne2022kl}, and
$-\log(1-q)\leq q/(1-q)$ give
\begin{equation*}
 \beta_\ell^X\leq\frac12\sqrt{-\log(1-a^{2\ell})}
 \leq\frac{r^\ell}{2\sqrt{1-r^2}}.
\end{equation*}
Applying the same calculation to $Y$ yields
$\beta_\ell\leq r^{\ell-1}/\sqrt{1-r^2}$ for $\ell\geq2$.
All estimation assumptions therefore hold.
Theorem~\ref{thm:velocity-estimation}, with $\nu^2=2$ and
$C_{4\to2}^4=8$, now gives \eqref{eq:gaussian-risk}.

With the risk bound established, we turn to rollout error. Condition on
the fitted parameters
$(\widehat a,\widehat b)$ and write $R=R(\widehat{\vv},\vv^\star)$.
Couple $\widehat Y_0\sim\widehat\mu_0$ and $Y_0\sim\pi_Y$ so that
$\|\widehat Y_0-Y_0\|_{L^2}=a_0$, independently of the fresh stationary source.
The source noise variables $\epsilon_{i+1}=X'_{i+1}-aX'_i$ are iid standard
Gaussian, with $\epsilon_{i+1}$ independent of the histories through time $i$.
Using them to drive both recursions, \eqref{eq:gaussian-solution} gives
\begin{align*}
 \widehat Y_{i+1}
 &=\widehat b\widehat Y_i+(a-\widehat a)X'_i+\epsilon_{i+1},\\
 Y_{i+1}&=bY_i+\epsilon_{i+1}.
\end{align*}
The comparison trajectory has the stationary target law. Subtracting the
updates cancels the shared noise and gives
\begin{equation*}
 \widehat Y_{i+1}-Y_{i+1}
 =\widehat b(\widehat Y_i-Y_i)+(\widehat b-b)Y_i+(a-\widehat a)X'_i.
\end{equation*}
Minkowski's inequality, $|\widehat b|\leq r$, and \eqref{eq:gaussian-risk-identity} therefore give
\begin{align*}
 \|\widehat Y_{i+1}-Y_{i+1}\|_{L^2}
 &\leq r\|\widehat Y_i-Y_i\|_{L^2}
 +\frac{|\widehat b-b|}{\sqrt{1-b^2}}+\frac{|\widehat a-a|}{\sqrt{1-a^2}}\\
 &\leq r\|\widehat Y_i-Y_i\|_{L^2}+\sqrt{8R/\pi}.
\end{align*}
This estimate uses only the stationary second moments of $X'_i$ and $Y_i$;
they may be dependent under the coupling. Iterating from the initial cost
$a_0$ proves
\eqref{eq:gaussian-rollout}.
\end{proof}

\section{Examples for the conditional noise assumption}\label{app:noise-examples}

Bounded displacements, bounded additive noise, and Gaussian transition kernels
satisfy Assumption~\ref{ass:conditional-noise} under the conditions below.
Each calculation conditions on the full regression input
$\vh=(t,\vz_t,\vx,\vy)$ and verifies only the conditional noise assumption.

\paragraph{Bounded displacements.}
Recall $\vr=\vy^+-\vx^+$ and
$\boldsymbol\xi=\E[\vr\mid\vh]-\vr$ from Section~\ref{sec:estimation}.
Suppose $\norm{\vr}\leq B$ almost surely. For every fixed direction
$\va\in\R^d$, the scalar $\ip{\va}{\vr}$ lies in
$[-B\norm{\va},B\norm{\va}]$, including after conditioning on $\vh$.
Since $\ip{\va}{\boldsymbol\xi}=\E[\ip{\va}{\vr}\mid\vh]-\ip{\va}{\vr}$,
conditional Hoeffding's lemma \citep[Section~2.2]{vershynin2018high} gives
\begin{equation*}
 \E\left[e^{\theta\ip{\va}{\boldsymbol\xi}}\mid\vh\right]
 \leq\exp\left(\frac{\theta^2(2B\norm{\va})^2}{8}\right)
 =\exp\left(\frac{\theta^2B^2\norm{\va}^2}{2}\right).
\end{equation*}
Thus $\nu=B$ suffices. In particular, if both processes take values in a
ball of radius $B$, then $\norm{\vy^+-\vx^+}\leq2B$, so $\nu=2B$
suffices for either discrete or continuous state distributions.

\paragraph{Bounded additive noise.}
Subtracting any function of $(\vx,\vy)$ from the displacement leaves
its conditional residual unchanged. This allows the preceding Hoeffding
bound to apply when only the additive noise is bounded.
Consider
\begin{equation*}
 \vx^+=\vf(\vx)+\boldsymbol\eta^X,\qquad
 \vy^+=\vg(\vy)+\boldsymbol\eta^Y,\qquad
 \norm{\boldsymbol\eta^X},\norm{\boldsymbol\eta^Y}\leq B
 \quad\text{almost surely},
\end{equation*}
where $\vf,\vg:\R^d\to\R^d$ may be nonlinear and the noise distributions
may depend on the current states.
Conditional on $\vh$, the term $\vg(\vy)-\vf(\vx)$ is fixed. Consequently,
\begin{equation*}
 \boldsymbol\xi
 =\E[\boldsymbol\eta^Y-\boldsymbol\eta^X\mid\vh]
   -(\boldsymbol\eta^Y-\boldsymbol\eta^X).
\end{equation*}
The noise difference $\boldsymbol\eta^Y-\boldsymbol\eta^X$ has norm at most $2B$, so the preceding
calculation gives $\nu=2B$. The states and transition means themselves
need not be bounded.

\paragraph{Gaussian transitions.}
Suppose the two transition kernels are Gaussian with arbitrary mean
functions and covariance matrices bounded above by $\sigma^2\Id$, uniformly
in the current states. Fix $(\vx,\vy,t)$. The endpoints are independent
under \eqref{eq:conditional-endpoints}, and $(\vr,\vz_t)$ is jointly Gaussian.
Gaussian conditioning therefore gives a centered Gaussian residual with
\begin{equation*}
 \operatorname{Var}(\ip{\va}{\boldsymbol\xi}\mid\vh)
 =\operatorname{Var}(\ip{\va}{\vr}\mid\vh)
 \leq\operatorname{Var}(\ip{\va}{\vr}\mid\vx,\vy,t)
 \leq2\sigma^2\norm{\va}^2.
\end{equation*}
The Gaussian moment-generating function yields \eqref{eq:noise} with
$\nu=\sqrt2\sigma$, including for nonlinear conditional means.

For other sub-Gaussian transition models, Assumption~\ref{ass:conditional-noise}
requires a bound conditional on the full input $\vh$, including
$\vz_t=(1-t)\vx^++t\vy^+$. Sub-Gaussian bounds on the stationary states,
or on additive transition noise conditional only on the current states,
do not by themselves give the conditional bound used in
Lemma~\ref{lem:averaged-excess}.

\section{Experimental details}\label{app:experiments}

The Gaussian and nonlinear experiments use seeds $0,1,2$, with $d=10$
for the main panels and $d=100$ as a stress test. The main panels use
$m\in\{512,1024,2048,4096,8192,16384,32768\}$ training transitions;
the stress test uses $m\in\{256,512,1024,2048,4096,8192\}$.
Architecture and optimization settings stay fixed across training sizes.
Each seed produces one complete experiment run. We report the mean across
these three runs, with standard error equal to the sample standard deviation
divided by $\sqrt{3}$.
Each training trajectory starts from zero, discards 1,000 states, and retains
32,769 states at $d=10$ or 8,193 at $d=100$.
Training sets are nested prefixes; process parameters and
evaluation draws stay fixed across training sizes within each seed.
Independent random streams generate parameters, source and target training
trajectories, evaluation contexts, projection directions, and evaluation noise.

At each training update, we sample transition indices uniformly with
replacement and draw one fresh independent $t\sim\Unif(0,1)$ per sampled
transition. Each repeated occurrence of a transition receives a new flow
time, including repeated indices within a minibatch. Source and target
transitions retain their paired indices. The minibatch loss is an unbiased
estimate, conditional on the training trajectories and current network, of
$\widehat L(\vv)/d$, where \eqref{eq:integrated-empirical-loss} defines the
time-integrated objective. The factor $1/d$ averages over coordinates.
The velocity network takes $(t,\vz,\vx,\vy)$ and outputs a vector in $\R^d$.
It has three SiLU hidden layers, width 64 at $d=10$ and 256 at $d=100$,
and a zero-initialized final layer. Training uses 2,400 AdamW updates with
batch size 256, initial learning rate $2\times10^{-3}$, cosine decay,
weight decay $10^{-5}$, and gradient norm clipping at 5.
We compare the full conditional flow with two learned baselines.
All three learned flows share initialization, minibatch indices, and flow
times. The source-state ablation sets $\vx_i$ to zero. Training still uses
$\vx_{i+1}$ to form the interpolation and conditions on $\vy_i$;
generation starts from $\vx'_{i+1}$ and conditions on $\widehat{\vy}_i$.
The marginal-flow baseline zeros both context inputs and fits an
unconditional flow from the same source and target next-state observations.
Its ODE defines a single map $T$ applied pointwise as
$\widehat{\vy}_{i+1}=T(\vx'_{i+1})$. It is trained to match the marginal state
distributions and uses the same architecture,
training budget, and solver. Every learned flow uses the final training
update; no test metric selects a checkpoint or hyperparameter.
For a fixed current-state pair $(\vx,\vy)$, generation initializes
$\vz_0=\vx^+$ and applies 32 explicit midpoint steps, with $h=1/32$
and $k=0,\ldots,31$:
\begin{equation}
 \va_k=\widehat{\vv}_{kh}(\vz_k,\vx,\vy),\qquad
 \vz_{k+1}=\vz_k+h\widehat{\vv}_{(k+1/2)h}(\vz_k+h\va_k/2,\vx,\vy).
 \label{eq:rk2-solver}
\end{equation}

We evaluate the next-state distribution at 128 fresh current-state pairs.
For each pair $c$, each method produces 512 next states. We project them onto
$J_{\rm dir}=64$ independent directions drawn uniformly from the unit sphere.
Let $\widehat q_{c,j,k}$ be the $k$th smallest generated projection in direction
$j$, and let $q_{c,j}(p)$ be the corresponding true target-transition quantile
at probability $p$. The reported error is
\begin{equation}
 \left(\frac{1}{128\cdot J_{\rm dir}\cdot512}
 \sum_{c=1}^{128}\sum_{j=1}^{J_{\rm dir}}\sum_{k=1}^{512}
 \left[\widehat q_{c,j,k}-q_{c,j}\left(\frac{k-1/2}{512}\right)\right]^2
 \right)^{1/2}.
 \label{eq:synthetic-shared-metric}
\end{equation}
Evaluation uses fresh samples independent of training, with current-state
pairs, projection directions, and source draws shared across methods.
For every method, we evaluate conditional-transition recovery at the supplied
$(\vx,\vy)$; marginal flow ignores these two context inputs.
Projected Gaussian quantiles are analytic; projected mixture quantiles are
computed by 55 bisection steps on their exact cumulative distribution functions.
Thus the metric does not require an independently sampled target reference.
The oracle error is constant across training sizes because its parameters
and evaluation draws are fixed.

\subsection{Gaussian transition experiments}\label{app:gaussian-experiments}

The processes have transitions
$P_X(\vx,\cdot)=N(\ma_X\vx,\mq_X)$ and
$P_Y(\vy,\cdot)=N(\ma_Y\vy,\mq_Y)$.
For each seed, all four matrices have independent random orthogonal eigenbases.
Source transition eigenvalue magnitudes are uniform on $[0.72,0.94]$ and
target magnitudes on $[0.15,0.58]$, with independent random signs.
Noise covariance eigenvalues $\lambda$ satisfy
$\log\lambda\sim\Unif(\log0.25,\log1.75)$.
The true-kernel oracle is
\begin{equation}
 \Phi^{\rm or}_{\vx,\vy}(\vz)
 =\ma_Y\vy+\mq_Y^{1/2}\mq_X^{-1/2}(\vz-\ma_X\vx).
 \label{eq:gaussian-oracle-benchmark}
\end{equation}
It maps $N(\ma_X\vx,\mq_X)$ to $N(\ma_Y\vy,\mq_Y)$.

\subsection{Nonlinear mixture benchmark}\label{app:nonlinear}

For independent random orthogonal matrices $\mR_X,\mR_Y$, define
\begin{equation*}
 \vf(\vx)=0.85\mR_X\tanh(\mR_X^\top\vx),\qquad
 \vg(\vy)=\mR_Y\left[0.55\tanh(\mR_Y^\top\vy)
                  +0.2\sin(2\mR_Y^\top\vy)\right],
\end{equation*}
where the nonlinearities act coordinatewise. Let
$\vu=d^{-1/2}\mR_Y(1,\ldots,1)^\top$, so $\norms{\vu}=1$.
The source transition is $P_X(\vx,\cdot)=N(\vf(\vx),0.6^2\Id)$, and
\begin{equation*}
 P_Y(\vy,\cdot)=\tfrac12N(\vg(\vy)-0.8\sqrt d\,\vu,0.25^2\Id)
              +\tfrac12N(\vg(\vy)+0.8\sqrt d\,\vu,0.25^2\Id).
\end{equation*}
There is one mixture sign per noise vector, giving two multivariate
modes. The rotations couple the observed coordinates in both conditional means.

Let $F_{\rm N}$ be the standard normal cumulative distribution function and
$H_d$ that of $\tfrac12N(-0.8\sqrt d,0.25^2)+\tfrac12N(0.8\sqrt d,0.25^2)$.
For a supplied source next state $\vz$, set
$\vw=(\vz-\vf(\vx))/0.6$ and $s=\vu^\top\vw$.
The oracle is
\begin{equation}
 \Phi^{\rm or}_{\vx,\vy}(\vz)
 =\vg(\vy)+0.25(\vw-s\vu)+\vu H_d^{-1}(F_{\rm N}(s)).
 \label{eq:nonlinear-oracle-benchmark}
\end{equation}
Under the source transition, $s\sim N(0,1)$ independently of
$\vw-s\vu\sim N(0,\Id-\vu\vu^\top)$. The quantile map gives the required
mixture along $\vu$, while the orthogonal component supplies independent
Gaussian noise, yielding exactly $P_Y(\vy,\cdot)$.
For a unit evaluation direction $\va$, the projected target law is a mixture
with means $\va^\top\vg(\vy)\pm0.8\sqrt d\,\va^\top\vu$ and variance $0.25^2$
in each component; its quantiles enter \eqref{eq:synthetic-shared-metric}.

\subsection{Results at dimension 100}\label{app:highdim-stress}

We repeat both sample-size experiments at $d=100$ with width 256 and all
other training and evaluation settings unchanged. At $m=8192$, the mean
next-state errors (standard errors in parentheses) are
\begin{center}
\begin{tabular}{lcccc}
\toprule
Benchmark & Conditional flow & Without $\vx_i$ & Marginal flow & Oracle \\
\midrule
Gaussian & $0.733\ (0.015)$ & $1.043\ (0.024)$ & $1.032\ (0.016)$ & $0.070\ (0.001)$ \\
Nonlinear mixture & $0.565\ (0.027)$ & $0.666\ (0.021)$ & $0.829\ (0.009)$ & $0.129\ (0.006)$ \\
\bottomrule
\end{tabular}
\end{center}
At this sample size, the full conditional flow has lower mean error than
both learned baselines in both benchmarks. Its improvement over the
source-state ablation is smaller for nonlinear transitions. Errors remain above
the oracle, and the nonlinear ablation performs better at some smaller
sample sizes under the fixed training budget.

\section{Ten-subject EEG study}\label{app:enigma-cfm}

\subsection{Protocol}\label{app:eeg-protocol}

We evaluate all ten THINGS-EEG2 subjects using the released ENIGMA
implementation \citep{kneeland2026enigma,enigma2026code} under the matched
subject-adaptation protocol in Section~\ref{sec:experiments}.
ENIGMA uses subject-specific alignment and fine-tuning on calibration data
to adapt to differences between subjects
\citep[Sections~3.3 and~4.3]{kneeland2026enigma}.
Other EEG retrieval methods address these differences through label-free
calibration or coordinate alignment \citep{huang2026sattc,cui2026score}.
Each recording contains 250 time samples from 63 channels over one second
following image presentation. ENIGMA maps recordings to image embeddings
and ranks candidate images by embedding similarity.
Split seed 20260904 selects 4,096 calibration and 1,000 validation images.
Evaluation averages 80 EEG repetitions per image in the official
200-image gallery. Baseline and flow use identical splits and seeds.
This 63-channel adaptation protocol differs from ENIGMA's published
32-channel joint-training experiment \citep[Sections~3.3--4.1]{kneeland2026enigma}.

For adapted subject $s$, ENIGMA is pretrained on reference
subjects $\mathcal S_s$: subjects 5--10 for $s\in\{1,2,3,4\}$;
subjects $\{1,2,3,4,7,8,9,10\}$ for $s\in\{5,6\}$; and subjects 1--6
for $s\in\{7,8,9,10\}$. The baseline and both flow widths share this
pretrained model, whose encoder and projection network remain frozen.
We average seeds 0, 1, and 42 within each subject. Confidence intervals are
percentiles 2.5 and 97.5 over 100,000 paired bootstrap resamples of ten
subjects. For Figure~\ref{fig:enigma-main}, each resample recomputes the mean
CFM and baseline accuracies $\bar a$ and $\bar b$ and their relative lift
$100(\bar a/\bar b-1)\%$; for the absolute gains below, it computes
their difference. All configurations use the final training
update; no test value selects a checkpoint, width, or context length.

\subsection{Adapter configurations}\label{app:eeg-adapters}

\paragraph{Inputs and outputs.}
Let $\mE_s(u)\in\R^{63\times250}$ be subject $s$'s EEG epoch for image $u$,
$F_{\vtheta}$ the flow map with learned parameters $\vtheta$, and $D$ the frozen ENIGMA encoder and
projection network. The flow precedes the subject-specific temporal map
$\mw_s$; its endpoints and predicted embedding are
\begin{equation}
 \ma(u)=\mE_s(u),\qquad \mb_q(u)=\mE_q(u)\quad(q\in\mathcal S_s),
 \qquad D\left(\mw_s\left(F_{\vtheta}(\ma(u))\right)\right).
 \label{eq:eeg-endpoints}
\end{equation}
The map $\mw_s$ has a $250\times250$ weight matrix and 250-dimensional
bias shared across channels, initialized by averaging the reference subjects'
pretrained maps. We train one flow and one map for each adapted subject,
seed, width, and context length. The flow outputs a $63\times250$ epoch;
$D$ maps the adapted epoch to an image embedding.

\paragraph{Past-state conditioning.}
For any epoch $\mE$ with columns $\mE_j\in\R^{63}$, define the ordered
history vector $\vh_k(\mE,j)=(\mE_{j-1},\ldots,\mE_{j-32})$, with each
slot replaced by zero if its lag exceeds $k$ or its time index is at most
zero. Thus $k\in\{0,1,4,16,32\}$ specifies the number of previous samples
available, while the input dimension stays fixed. The velocity
$\vv_{\vtheta}(t,\vz;\vh_k(\ma,j),\vh_k(\mb_q,j))\in\R^{63}$ receives
flow time $t$, the current interpolation $\vz$, and both histories.
Its implementation concatenates one zero 63-channel slot, the source
history, $\vz$, the target history, and $t$, giving 4,159 input coordinates.
The current source endpoint is used to form the interpolation but is not
separately supplied to the velocity.
A network shared across all EEG samples has two hidden GELU layers, each
of width 64 for our main model or 14 for the smaller-model comparison,
and a 63-dimensional linear output; all layers have biases.
There is no normalization, dropout, skip, or recurrent layer.
Zero output-layer initialization makes the initial flow the identity.

\paragraph{One source and multiple reference targets.}
For a batch of $N=64$ distinct calibration images $u_b$, write
$\ma_b=\ma(u_b)$ and $\mb_{bq}=\mb_q(u_b)$, with columns $\va_{bj}$ and
$\vb_{bqj}$. Pair each source epoch separately with every $q\in\mathcal S_s$.
Draw independent $t_{bq}\sim\Unif(0,1)$, shared across the 250 samples of
that pair, and set $\vz_{bqj}=(1-t_{bq})\va_{bj}+t_{bq}\vb_{bqj}$ and
$\vr_{bqj}=\vb_{bqj}-\va_{bj}$.

Let $\Lenigma$ denote the supervised
embedding and contrastive objective described in the optimization paragraph
below. We add the flow-matching loss $\Lfm$ to obtain
\begin{equation}
 \begin{split}
 \Leeg&=\Lenigma+0.02\Lfm,\\
 \Lfm&=\frac{1}{N|\mathcal S_s|\,63\cdot250}
 \sum_{b=1}^{N}\sum_{q\in\mathcal S_s}\sum_{j=1}^{250}
 \norm{\vv_{\vtheta}\left(t_{bq},\vz_{bqj};
 \vh_k(\ma_b,j),\vh_k(\mb_{bq},j)\right)-\vr_{bqj}}^2.
 \end{split}
 \label{eq:eeg-main-loss}
\end{equation}
Reference subjects receive equal weight, with one loss term for each
source--reference pair. The same
velocity network handles all references, without a reference-subject identity
input. Training conditions on the observed reference history
$\vh_k(\mb_{bq},j)$, held fixed as flow time varies.

\paragraph{Sequential generation.}
For a source epoch $\ma$ with columns $\va_j$, generate columns
$\widehat{\vb}_j$ of $\widehat{\mb}=F_{\vtheta}(\ma)$ in order
$j=1,\ldots,250$ by approximately solving
\begin{equation}
 \vz_j(0)=\va_j,\qquad
 \frac{\dd\vz_j(t)}{\dd t}
 =\vv_{\vtheta}\left(t,\vz_j(t);
 \vh_k(\ma,j),\vh_k(\widehat{\mb},j)\right),\qquad
 \widehat{\vb}_j=\vz_j(1).
 \label{eq:eeg-generation}
\end{equation}
Two explicit-midpoint steps of size $1/2$ integrate each local ODE, holding
both histories fixed. Target history contains only previously generated
samples; at $j=1$, both histories are zero. Thus the initial sample uses a
learned marginal flow, and generation requires no reference recording or
reference-subject choice. The supervised loss $\Lenigma$ differentiates
through the full sequential generation and $\mw_s$; $D$ remains frozen.
For $k=1$ and $j>1$, the velocity inputs and source initialization follow
our theoretical transition construction. Here training pairs recordings of
the same image, and generation uses a
learned flow for the initial sample. The theory assumes independent source
and target training trajectories and stationary target initialization.
We do not assume that the EEG data satisfy its stationary Markov or mixing
hypotheses.

\paragraph{Optimization.}
We retain ENIGMA's mean squared error over image-embedding coordinates and
its symmetric image and text contrastive losses, weighted $1$, $0.5$, and
$0.05$, respectively
\citep{enigma2026code}. Contrastive scores use normalized target embeddings,
unnormalized predictions, and the fixed pretrained score multiplier;
the squared loss uses unnormalized image embeddings.
The baseline trains only $\mw_s$; our method jointly trains $\mw_s$ and the flow.
Both use 300 AdamW updates, peak learning rate $3\cdot10^{-4}$, weight decay
$0.01$ excluding biases and normalization parameters, gradient clipping at
norm 1, and a one-cycle schedule with 10\% warmup and cosine decay.
Each update draws a fresh image batch and fresh flow times.

\subsection{Additional results}\label{app:eeg-throughput}

Table~\ref{tab:enigma-baseline} reports the $k=32$ comparison with the flow
before $\mw_s$.
Our main width-64 model achieves $49.05\%$ mean Top-5 accuracy, versus
$43.87\%$ for the baseline: a gain of $5.18$ percentage points (95\%
interval $[4.08,6.08]$), or $11.82\%$ relative lift ($[9.25,14.06]\%$).
The smaller width-14 model reaches $46.73\%$, a gain of $2.87$ points
($[2.13,3.53]$), or $6.53\%$ relative lift ($[4.81,8.22]\%$).
Both improve on the baseline for all ten subjects after averaging seeds.
Figure~\ref{fig:enigma-main} reports every context length under the same
protocol. Longer history improves retrieval at both widths without
increasing model size within either sweep.

\begin{table}[ht]
 \centering\small
 \caption{Retrieval accuracy (\%) at $k=32$, averaging seeds within subjects
 and then ten subjects. Both CFM networks precede $\mw_s$.}
 \label{tab:enigma-baseline}
 \begin{tabular}{lrrr}
 \toprule
 Adapter&Top-1&Top-5&Top-10\\\midrule
 $\mw_s$ only&17.78&43.87&57.40\\
 CFM, width 64&20.55&49.05&62.40\\
 CFM, width 14&18.60&46.73&59.82\\
 \bottomrule
 \end{tabular}
\end{table}

\subsection{Flow after the temporal map}\label{app:eeg-post-ws}

We also place the flow after $\mw_s$, giving the predicted embedding
$D(F_{\vtheta}(\mw_s(\mE_s(u))))$. The subjects, reference sets, data splits,
seeds, architectures, optimization settings, and final checkpoint rule are
unchanged. The source and reference endpoints are now
$\ma(u)=\mw_s(\mE_s(u))$ and $\mb_q(u)=\mw_q(\mE_q(u))$.
Each reference map $\mw_q$ is frozen at its pretrained value; $\mw_s$ and
$F_{\vtheta}$ train jointly, with gradients through $\mw_s$ in both the
retrieval and flow-matching losses. The velocity inputs, loss
\eqref{eq:eeg-main-loss}, and sequential generation
\eqref{eq:eeg-generation} use these transformed endpoints.
Thus $k$ counts preceding samples in the transformed sequence. Since the
temporal maps mix all 250 raw time samples, this history need not correspond
to preceding measurements in raw EEG time.

Figure~\ref{fig:eeg-post-ws} reports all five context lengths for both widths.
Relative Top-5 lift increases from $3.91\%$ at $k=0$ to $11.63\%$ at
$k=32$ for width 64, and from $1.44\%$ to $7.26\%$ for width 14.
At $k=32$, the width-64 model achieves $48.97\%$ Top-5 accuracy, a gain
of $5.10$ percentage points over the $43.87\%$ baseline (95\% interval
$[4.05,6.28]$), or $11.63\%$ relative lift ($[8.83,15.00]\%$).
Width 14 reaches $47.05\%$, a gain of $3.18$ points ($[2.10,4.33]$),
or $7.26\%$ relative lift ($[4.61,10.37]\%$).
Both improve on the baseline for all ten subjects after averaging seeds.
These results show that the benefit of additional context also appears
when the flow acts on temporally transformed EEG. Table~\ref{tab:eeg-post-ws}
gives the corresponding Top-1, Top-5, and Top-10 accuracies.

\begin{figure}[ht]
 \centering
 \includegraphics[width=0.8\linewidth]{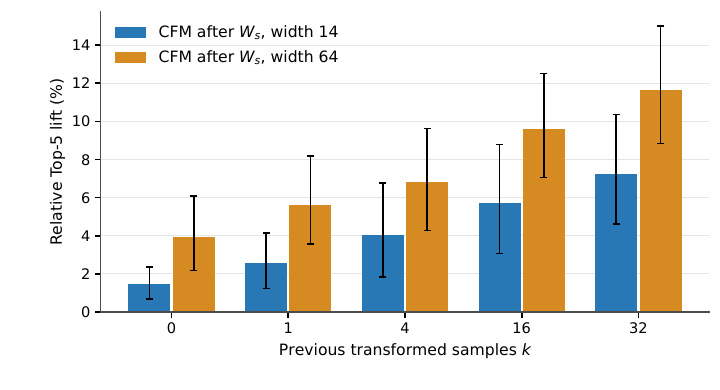}
 \caption{Additional EEG comparison with both flow networks after $\mw_s$.
 Relative Top-5 retrieval accuracy lift over the matched $\mw_s$-only baseline
 is plotted against context length $k$ in transformed EEG coordinates.
 Accuracies average three seeds within each of ten subjects; error bars
 give 95\% paired subject-bootstrap percentile intervals.}
 \label{fig:eeg-post-ws}
\end{figure}

\begin{table}[ht]
 \centering\small
 \caption{Retrieval accuracy (\%) at $k=32$ with CFM after $\mw_s$,
 averaging three seeds within subjects and then ten subjects.}
 \label{tab:eeg-post-ws}
 \begin{tabular}{lrrr}
 \toprule
 Adapter&Top-1&Top-5&Top-10\\\midrule
 $\mw_s$ only&17.78&43.87&57.40\\
 CFM after $\mw_s$, width 64&20.65&48.97&62.23\\
 CFM after $\mw_s$, width 14&19.38&47.05&60.32\\
 \bottomrule
 \end{tabular}
\end{table}

\section{Further comparisons with prior work}\label{app:related-work}

\subsection{Sequential transport and generation}\label{app:related-transport}
Time-series adaptation methods align observations, time axes, or
representations across domains \citep{painblanc2023match,mishra2026causal}.
Flow models generate entire functions or correlated sequences
\citep{kerrigan2024functional,wei2025stream}; Schr\"odinger-bridge methods
learn dynamics between endpoint laws or complete time-series laws
\citep{tong2024schrodinger,hamdouche2026nonparametric}.
COT-GAN uses causal transport as a training objective for sequential data
\citep{xu2020cotgan}. We instead transform a prescribed fresh source
trajectory by recursively applying a learned conditional transition map.

Causal transport requires
$\Law(\vy_{0:i}\mid\vx_{0:n})=\Law(\vy_{0:i}\mid\vx_{0:i})$;
bicausal transport imposes the reverse identity as well
\citep{backhoff2017causal,moulos2021bicausal}.
For Markov chains, \citet{moulos2021bicausal} compute bicausal transport costs
by value iteration, and \citet{oconnor2022optimal} optimize stationary
transition couplings for long-run average cost.
\citet{oconnor2024estimation} give finite-sample cost bounds for estimated
stationary couplings of finite-valued processes under mixing assumptions.
Our statistical bound controls the squared prediction error of a continuous
conditional velocity; \eqref{eq:rollout-from-risk} converts it into
accuracy of generated states.

\subsection{Statistical error and dependence}\label{app:related-statistics}
\citet[Theorem~1]{benton2024error} and
\citet[Proposition~6.2]{gao2024gaussian} convert population velocity error
into Wasserstein error. For stochastic diffusion flow matching,
\citet[Theorems~2--3]{silveri2024theoretical} bound Kullback--Leibler error;
\citet[Theorems~1--2]{bansal2026straightness} and \cite{roy2026low} analyze rectified-flow
discretization. We combine velocity estimation from dependent transitions
with conditional Wasserstein bounds and recursive generation.
Our estimation proof uses the coupling and fourth-moment approach to
dependent square-loss regression, as in
\citet[Theorem~3.1]{roy2021dependent}.

For convex or realizable predictor classes,
\citet[Theorem~3.1 and discussion following (5)]{ziemann2024sharp} show that,
under suitable moment assumptions, the leading square-loss error term need
not contain a mixing-time factor. Dependence enters the minimum sample size
and higher-order terms. Our lower bound is consistent with this distinction:
it establishes a mixing-dependent sample threshold.

\subsection{Distribution mismatch and rollout stability}\label{app:related-rollout}
In offline policy evaluation,
\citet[Section~4.3, Lemma~2, (16b)]{zanette2022bellman} compare squared
Bellman-residual norms under the target-policy occupancy law $d^\pi$ and
the data law $\mu$. If $d^\pi$ has density at most $M$ relative to $\mu$,
then every $f$ with $\E_{(s,a)\sim\mu}f(s,a)^2<\infty$ satisfies
\begin{equation*}
 \E_{(s,a)\sim d^\pi}f(s,a)^2
 \leq M\E_{(s,a)\sim\mu}f(s,a)^2,
\end{equation*}
where $(s,a)$ is a state--action pair. Their application takes $f$ to be
a Bellman residual.
\citet[Section~1]{duan2020minimax} similarly restrict distribution mismatch
to a function class through a restricted $\chi^2$ divergence.
The concentrability condition \eqref{eq:rollout-concentrability} compares our conditional
velocity error under the law $\Gamma_i^{\vv}$ of source and generated state
pairs and the training law $\pi_X\otimes\pi_Y$, with the conditional risk
$\ell_{\vv}$ from \eqref{eq:conditional-velocity-risk} playing the role of
the squared Bellman residual.

\citet[Section~3.1, Theorem~3.1]{rudolf2018perturbation} propagate
one-transition discrepancies through contracting Markov dynamics in
first-order Wasserstein distance. Proposition~\ref{prop:rollout} gives a
conditional $W_2$ coupling that retains the prescribed source transitions
and controls both state and joint-trajectory error.

\end{document}